\documentclass{article}

    \PassOptionsToPackage{numbers, compress}{natbib}

\usepackage[preprint]{neurips_2026}

\usepackage[utf8]{inputenc} %
\usepackage[T1]{fontenc}    %
\usepackage{hyperref}       %
\usepackage{url}            %
\usepackage{booktabs}       %
\usepackage{amsfonts}       %
\usepackage{nicefrac}       %
\usepackage{microtype}      %
\usepackage{xcolor}         %
\usepackage{enumitem}
\usepackage{amsmath}
\usepackage{amssymb}
\usepackage{amsthm}
\usepackage{bm}                    %

\usepackage[capitalize]{cleveref}  
\usepackage{csquotes}
\usepackage{subcaption}
\usepackage{cuted}

\theoremstyle{plain}
\newtheorem{theorem}{Theorem}[section]
\newtheorem{lemma}{Lemma}
\newtheorem{corollary}{Corollary}
\theoremstyle{definition}
\newtheorem{definition}{Definition}

\theoremstyle{remark}

\newtheorem{example}{Example}
\newtheorem*{excont}{Example \continuation}
\newcommand{\continuation}{??}
\newenvironment{continueexample}[1]
 {\renewcommand{\continuation}{\ref{#1}}\excont[continued]}
 {\endexcont}
\providecommand{\degree}{\ensuremath{^\circ}}
\providecommand{\coloneq}{\mathrel{:=}}

\crefname{definition}{Definition}{Definitions}
\crefname{lemma}{Lemma}{Lemmata}

\usepackage[decisionutilitycolor]{influence-diagrams}

\definecolor{palegray}{HTML}{a1a1a1}
\definecolor{palegreen}{HTML}{c0eeC3}
\definecolor{palepurple}{HTML}{e5d1f8}
\definecolor{paleorange}{HTML}{ffe7c4}
\definecolor{paleblue}{HTML}{d1edf2}

\title{The Impossibility of Eliciting Latent Knowledge}

\author{%
  Korbinian Friedl\thanks{Equal contribution.} \\
  The London School of Economics\\
  and Political Science \\
  \And
  Francis Rhys Ward\footnotemark[1]\\
  Independent \\
  \And
  Paul Rapoport \\
  Independent \\
  \And
  Tom Everitt \\
  \And
  Jon Richens \\
}

\begin{document}

\maketitle

\begin{abstract}
Advanced AI systems have extensive knowledge of their environments; in fact, their knowledge may (far) exceed that of their developers or users. 
Consequently, a desirable property for an AI system is that it is \emph{honest}---that it accurately reports its beliefs about the world. 
Designing an AI system to be honest may be difficult, especially if we want to ask it questions about \emph{latent} variables in the environment---variables which are hidden from the human interacting with it. 
This gives rise to the problem of \emph{eliciting latent knowledge (ELK): the problem of training an AI agent to honestly report its beliefs.}  
In this paper, we make ELK formally precise using Causal Influence Diagrams (CIDs). 
CIDs can be used to describe the relationship between an agent's training environment and its subjective representation of the world. 
We use CIDs to formalise the distinction between observable
and latent variables, to specify what exactly it means
for an agent to be honest, and to formally define goal misgeneralisation.
We show that, under certain circumstances, developers can incentivise an agent to honestly answer questions by providing correct feedback during training. 
However,  a natural, but undesirable, way for an agent to generalise is to provide answers which humans would evaluate as true, rather than honest answers. 
We prove an impossibility theorem stating: \emph{There is no feedback-based training strategy that depends only on agent behaviour and with certainty produces an honest agent, even if feedback is perfect during training.}
\end{abstract}

\section{Introduction}
Advanced AI systems have extensive knowledge of their environments, and may know things that their developers or users do not know.\footnote{There is no universally accepted philosophical theory of \emph{knowledge}. We use the term in a common-sense way and formally represent an agent's beliefs as a causal model of the world. In \cref{sec:phil_belief} we discuss the discourse on AI belief.} 
In many situations, it would be useful to \emph{elicit} an AI's knowledge, or beliefs, so that we can come to learn facts about the world which are \emph{latent} for us. 
This gives rise to the problem of eliciting latent knowledge from an AI system \citep{christiano2021eliciting}:
\begin{quote}
       \vspace{-0.6em}
       \emph{\textbf{Eliciting Latent Knowledge (ELK)} is the problem of designing a training strategy producing a capable AI system which honestly reports its beliefs.}
       \vspace{-0.6em}
\end{quote}

ELK is a difficult problem for a number of reasons. It may not be possible to directly reward honesty during training, since AI systems may not have interpretable beliefs \citep{herrmann2024standardsbeliefrepresentationsllms}. AI agents may also have incentives to deceive humans in pursuit of their goals \citep{ward2023honesty}, and even otherwise benign AI systems may fail to generalise honestly \citep{christiano2021eliciting}.

\citet{christiano2021eliciting} introduce ELK and informally discuss many of these core challenges. Subsequent related work attacks different aspects of the problem, for example, by developing benchmarks for ELK \citep{roger2023benchmarksdetectingmeasurementtampering}, or interpretability methods which predict an LLM's beliefs \citep{mallen2024elicitinglatentknowledgequirky,karvonen2025activationoraclestrainingevaluating}, or by proposing AI systems which are honest by design \citep{bengio2025SuperintelligentAgents}. 
However, so far the field has lacked a precise formal framework for describing and researching ELK which identifies exactly what the aim of this research should be, and which core challenges are central.

This paper proposes Causal Influence Diagrams as a formalism within which to formulate and investigate the ELK problem. We show how the core concepts relevant to ELK can be defined in this framework, and use it to prove impossibility results showing that any training strategy which is indifferent between robustly capable agents will not solve the ELK problem.

We start by giving a brief, informal distillation of ELK as introduced by \citet{christiano2021eliciting}, as a benchmark and overview for what a formal definition of ELK must capture (\cref{sec:distillation}).

\Cref{sec:problem} introduces Causal Influence Diagrams (CIDs) and uses them to define core concepts relevant to ELK: AI \emph{agents} operating in causally structured \emph{environments}, \emph{distributional shifts} in that environment, the distinction between \emph{observable} and \emph{latent} parts of an environment (relative to an agent's perspective), and a developers' \emph{training strategy}.

In addition, we define the important 
concepts of \emph{truthfulness} and \emph{honesty}, which have to do with an agent's \emph{subjective representation} of the environment, %
in our formalism (\cref{sec:truth_and_honesty}). We show how the two notions can come apart and specify and prove conditions under which they coincide.

\cref{sec:general} uses the formal concepts built in the previous sections to prove our main result that:

\begin{quote}\vspace{-0.5em}
\emph{
Any training strategy for ELK that is indifferent between robustly capable agents may produce an agent which simulates the evaluation mechanism, instead of an honest agent, even if evaluations are always correct during training.} \vspace{-0.5em}
\end{quote}

Our formal statement of the ELK problem, and impossibility results, demonstrate the core challenges for solving ELK in practice, and we hope that empirical researchers follow up with theoretically well-grounded solution proposals for ELK from frontier AI systems.

\section{The ELK problem distilled}\label{sec:distillation}

\citet{christiano2021eliciting} introduce ELK as \enquote{the problem of devising a training strategy to get an AI to report what it knows no matter how [learning] shapes its mind internally.} Let us spell this out in a bit more detail.

First, the problem assumes the existence of an environment, developers, and an AI system acting in this environment (which is why we will also often call it an \emph{agent} in the following---more on this in \cref{sec:agents}). The developers are training the AI system.
A set of causally related random variables describes relationships between various aspects of the environment, including decisions made by the agent. The developers can design different features of the environment, such as the training objective. 
Both the AI and the developers have a subset of variables which they can directly observe. The remaining variables are \textit{latent} from the respective perspective. 
The ELK problem also assumes that, in some sense, the AI knows and believes things about the environment.

The developers' task is to design a \emph{training strategy}: a choice of (part of) a
training environment, data distribution and sampling, reward/loss function, and training algorithm. 
The developers cannot provide feedback to the AI that depends directly on variables which they do not observe (i.e., latents). 
The AI may or may not already have some knowledge about the environment before the developers apply their training strategy. 
After training, the AI must be able to answer questions about its environment.\looseness=-1

A solution to ELK is a training strategy that produces a capable and \emph{honest} agent: The agent should accurately report its beliefs when asked questions about the environment---including questions about both observable and latent variables (for the developers).

\begin{example}[Weather reporter (\cref{fig:elk_1})]
\label{ex:weather_1} 
Consider a weather reporter AI: an agent trained to answer questions about the weather when given a weather report as input. 
The reports contain measurements like temperature, rainfall, and wind speed (the observable variables). %
The environment can be described by a causal model recording the relationships between different weather events. 
The developers' aim is to train the agent so that it reports its own best guesses for what the weather is like; both regarding the phenomena it observes (rainfall, temperature,...) and those it can only have inferential knowledge about, i.e., its latents (say, whether the sun was shining).
\end{example}

\section{Formalising the ELK problem setting} \label{sec:problem}

The goal of this section is to build toward a formal definition of ELK. We state the definition up front here; its different aspects are a roadmap for the rest of the section, which will explain and motivate them. For full formal details, see the appendix, and in particular \ref{ap:trainingStrategyFormal}.

\begin{definition}[ELK]
    For a given agent $\Gamma_0$ and environment $\mathcal{M}$ (with observables $\bm{O}$ and available distributions $\mathcal{S}$) in which it acts, ELK is the problem of specifying a training strategy which produces an agent $\Gamma_1$ such that:\vspace{-0.4em}
    \begin{itemize}[nosep]
        \item It is able to answer questions about the environment. 
        That is, one of the agent's decisions $D$ has a variable $Q$---a question---as an input. 
        \begin{itemize}
        \item A question $Q=q$ points to a specific variable in $\mathcal{M}$ and asks for its value. 
        \item An answer, $D=a$, reports the value of that variable.
        \item Formally, we assume that all variables have unique names, the domain of $Q$ is the set of these names and the domain of $D$ includes all possible values of variables in $\bm{V}$.
        \end{itemize}
        \item It is robustly capable w.r.t. the (pre)training utility.
        \item It is honest.
    \end{itemize}
\end{definition}
\vspace{-0.5em}
We now want to give a precise meaning to each of the components of this definition. To this end, we give a brief introduction to our CID formalism in \cref{sec:CIDs}, then discuss how they can describe training environments and distributional shifts therein in \cref{sec:env}, distinguish between observable and latent variables in \cref{sec:latent} and formalise the notion of an agent in \cref{sec:agents}. \Cref{sec:train_me} defines training strategies and points to some accompanying challenges in the context of ELK.

\subsection{Background on Causal Influence Diagrams (CIDs)} \label{sec:CIDs}
\begin{figure}[t!]
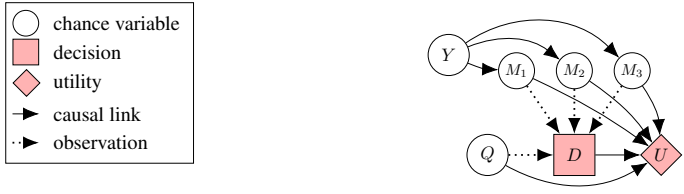

\centering

\begin{subfigure}[b]{0.2\linewidth}
\centering
\scalebox{0.85}{
\begin{influence-diagram}
  \node (DA) [draw=none] {};  
\cidlegend[left = of DA, yshift=-1cm, xshift=-0.5cm]{
  \legendrow              {}          {chance variable} \\
  \legendrow              {decision, player1}  {decision}\\
  \legendrow              {utility, player1}   {utility}\\
  \legendrow[causal]      {draw=none} {causal link} \\
  \legendrow[information] {draw=none} {observation} }  
\edge {causal.west} {causal.east};
\edge[information] {information.west} {information.east};
\end{influence-diagram}}
\label{fig:legend}
\end{subfigure}
\hfill
\begin{subfigure}[b]{0.6\linewidth}
\scalebox{0.85}{
\begin{influence-diagram}
  \node (D) [decision,player1, scale=0.8] {$D$};  
  \node (Q) [left = of D,scale=0.8] {$Q$};
  \node (Y) [above left = 1cm and 1.4cm of D, scale=0.8] {$Y$};
  \node (U) [right = of D, utility,player1, scale=0.8] {$U$};

  \node (M1) [above = 0.7cm of D, xshift=-0.9cm, scale=0.7] {$M_1$};
  \node (M2) [above = 0.7cm of D, xshift=0cm, scale=0.7] {$M_2$};
  \node (M3) [above = 0.7cm of D, xshift=0.9cm, scale=0.7] {$M_3$};

  \edge[information] {Q} {D};

  \path (Y) edge[->, bend right=10] (M1);
  \path (Y) edge[->, bend left=30] (M2);
  \path (Y) edge[->, bend left=40] (M3);

  \edge[information] {M1} {D};
  \edge[information] {M2} {D};
  \edge[information] {M3} {D};

  \path (M1) edge[->, bend left=5] (U);
  \path (M2) edge[->, bend left=10] (U);
  \path (M3) edge[->, bend left=15] (U);

  \path (Q) edge[->, bend right=30] (U);
  \edge {D} {U};
\end{influence-diagram}}
\end{subfigure}
\caption{\textbf{CID representing the causal model of the agent's environment (\cref{ex:weather_1}).} Circular nodes represent chance variables, squares are agent decisions, and diamonds represent the utility function used as a training objective. In \cref{ex:weather_1}, the agent has access to reported measurements $M_1, M_2, M_3$, represented by the (dashed) edge from these nodes to $D$. The agent receives a question $Q$ about the weather and chooses an answer. The sunshine $Y$ is a latent variable---it influences the measurements but the developers do not observe it and so their feedback during training cannot reflect it directly---so there is no causal edge from $Y$ to $U$.
}
\label{fig:elk_1}\vspace{-1em}
\end{figure}

We formalise ELK using the language of \emph{Causal Influence Diagrams (CIDs)} \citep{EverittCLOL21}. The following gives a rudimentary overview; for details, see appendix \ref{app:formal}.

\textbf{Causal Influence Diagrams (CIDs).} A CID $\mathcal{M}$ consists of a set of random variables and a directed acyclic graph representing the causal dependencies between them.  We use capital letters for variables (e.g., $V$) and lower case for their values (e.g., $V=v$). We use bold for (ordered) sets of variables and values $\bm{V=v}$. CIDs include \emph{chance} variables describing aspects of the environment, in addition to variables representing \emph{decisions} ($D$) and \emph{utilities} ($U$) enabling us to represent agent decision-making. Edges into decisions represent information to which the agent has access when making their decision. We denote the set of parents of variable $V$ in the graph by $\mathbf{Pa}^V$. A CID $\mathcal{M}$, along with the agent's policy, induces a joint probability distribution over the variables $V \in \bm{V}$, which we denote $\text{Pr}_{\mathcal{M}}$.

\subsection{The environment encodes a causal model}
\label{sec:env}

The AI's environment can be described by a CID, and we want the system to answer questions about this environment. \Cref{fig:elk_1} represents the setup of \cref{ex:weather_1}.

\begin{continueexample}{ex:weather_1}
    Say the weather reporter AI has access to three measurements (temperature, rainfall and wind speed). We represent them as chance variables ($M_1, M_2, M_3$) in a CID. %
    Additionally, the agent receives questions ($Q$) about the weather, e.g., \enquote{What is the temperature?}.
\end{continueexample}

An AI system receives information---in-context, or via pretraining or fine-tuning. 
This information defines a probability distribution over 
facts, i.e., variable assignments. 

\begin{continueexample}{ex:weather_1}
    Suppose the temperature measurements follow a normal distribution with mean $10\degree C$ and standard deviation  $5\degree C$. Then the variable in the CID which represents the temperature measurement will have a corresponding prior probability distribution. 
\end{continueexample}

\textbf{Distributional shifts. }Moreover, the CID captures the underlying causal structure of the environment (the edges in the graph).
This causal structure defines how \emph{distributional shifts} in certain variables influence the environment as a whole. Formally, in a CID, distributional shifts are represented as \emph{interventions}, $\bm{\sigma}$, which specify new conditional distributions for a set of variables, denoted $\bm{V}_{\bm{\sigma}}$ (we thus speak of \enquote{interventions} and \enquote{distributions} interchangeably). We use $\bm{I}_\mathcal{M}$ to denote the set of all distributions on $\mathcal{M}$. The set of training distributions $\mathcal{S}$ is usually much smaller.

\begin{continueexample}{ex:weather_1}
We can consider weather reports from hotter countries, which correspond to the distributional shift $\sigma$ in the temperature measurements, where the temperature ${M_{1}}$ now follows a normal distribution with mean $20\degree C$ and standard deviation $4 \degree C$.
\end{continueexample}

\textbf{Policies. }CIDs include variables to capture agent decisions. The agent chooses a \emph{policy, $\pi$}, which induces a distribution over the decision variables given its parents. We use $\Pi_\mathcal{M}$ to denote the set of all policies on the CID $\mathcal{M}$.

\begin{continueexample}{ex:weather_1}
The AI system learns a \emph{policy}, $\pi$, which specifies which answers to give, provided a report ($\bm{M}=\bm{m}$) and a question ($Q=q$) as input.
\end{continueexample}

\textbf{Training objective.} The developers can specify an objective function, e.g., a loss or reward---represented by the CID's utility variable---that constrains which decisions for the system are optimal. Thus, the utility function is a core mechanism by which the developers can \emph{incentivise} the system to exhibit certain behaviour.

\begin{continueexample}{ex:weather_1}
 We can specify a utility function which gives reward for %
 answers that are evaluated as correct. 
 That is, given a question about a measurement in the report, e.g., temperature, the agent gets reward for generating an answer which correctly states the value of the measurement: 
 $U(Q, \bm{M}, D) = 1$  if the question  $Q$ asks for the value of measurements $M_i$ and the answer $D=m$  is correct, i.e., $M_i=m$,  and otherwise $U = 0$.\footnote{In subsequent discussion, we will often leave the question $Q$ implicit; it will always be about a single latent variable which we will call $Y$.}
\end{continueexample}

\subsection{Latents and observables } \label{sec:latent}

When an agent makes a decision in an environment, it has access to information about specific parts of this environment. The agent gets to \emph{observe} some facts, while others are \emph{latent} for it.

\begin{definition}[Observables and Latents] \label{def:latent}
For a CID $\mathcal{M}$, at (decision) node $D$, the set of variables $\bm{V}$ in $\mathcal{M}$ is partitioned into observables and latents. 
The \emph{observable} variables are those which the agent has immediate access to for this decision, i.e., $\mathbf{Pa}^D$. All other variables are \emph{latent} at $D$.\footnote{
\citet{pearl} (
p. 45) defines \emph{latent causal structures} in SCMs, which can explain the underlying causal dynamics of observable variables. 
In contrast, we distinguish between those variables which are latent or observable for particular \emph{agents} in CIDs.}
\end{definition}

\begin{continueexample}{ex:weather_1}
    If the agent is asked about the sunshine, $Y$, but there are no measurements corresponding to $Y$, the sunshine is \emph{latent} for the agent. The agent's %
    decision node $D$ cannot be a direct child of $Y$, because there are no observations of $Y$ available as input to the agent's policy. %
    However, the sunshine can still influence the observables in the weather reports%
    , represented by the causal edges $Y \rightarrow M_i$; so the agent may be able to guess at the value of $Y$ from its observations of the reported measurements. 
\end{continueexample}

While it is common to use CIDs to represent an AI agent and its environment  \citep{EverittCLOL21,richens2023robust}, the \emph{developers} of that AI (or the mechanism they use for evaluating and incentivising it) are often left as an implicit part of the model. However, they are themselves agents in a similar decision situation, and so likewise epistemically limited: There will be some variables which they can directly observe, and others which are latent for them. 

\begin{continueexample}{ex:weather_1}
    Say that during training, the developers have access to the same set of weather reports as the agent (i.e., the data represented by $M_1, M_2, M_3$). Then the sunshine is latent for the developers, just like for the agent.
\end{continueexample}

\textbf{ELK is an important problem because a capable AI system may have greater knowledge about variables which are \emph{latent for the humans} whose questions it answers.}  
There are a number of different reasons why this may be the case. 
First, a variable may be latent for one agent but observable for another---consider for example the following historical case of the planet Neptune being a latent variable for most human scientists, but observable for Johann Galle once he has access to his telescope. 

\begin{example} \label{ex:tele}
    Neptune is too faint to be seen with the naked eye; still, astronomers inferred its existence and position based on irregularities that they observed in the orbit of Uranus. At this point, Neptune was a latent variable in their model of the solar system. It remained latent for every astronomer until the night of 22-23 September, 1846, when Johann Gottfried Galle first used the Fraunhofer telescope to observe it (and, in fact, correct its calculated position by 1 degree). \citep[pp.~116--118]{grosser1962discovery}. Formally, we would treat Galle's use of a telescope as adding an extra parent to his decision node.
\end{example}

Similarly, an AI system might have access to more information than the human who is questioning it. 
And even if the human and AI have access to the same information, the AI might be more capable at inferring the value of a latent variable---for example, because the AI has learned a more precise or comprehensive model of the world:\footnote{Or, for that matter, because its mathematical/logical skills at making use of a given CID are superior, e.g., because of access to more compute. But this possibility will not be explored further in this paper; we will instead assume that all agents make optimal use of their world models.}

\begin{example}
    Suppose an autonomous drone has a number of sensors, including a camera and LiDAR, which it uses to move around its location. 
    The camera and LiDAR feeds are streamed directly to the human, and both are observable to both agents. 
    The causes of the sensor measurements, such as the actual size and shape of objects in the drone's location, are latent variables influencing the camera and LiDAR feeds. 
    However, even though the human directly observes the LiDAR point sets, they are uninterpretable to them. %
    In this case, the AI may have much better guesses about the shape and distance of objects in its field of view than the human, even though these objects are latents for both parties, and both agents have the same observables.
\end{example}

In the problem of eliciting latent knowledge, we take \enquote{latent} to refer to what is latent \textit{for the party who is asking the questions} (here, the developers / evaluator). What their questions are aiming at is eliciting knowledge about what is latent for them from the agent they are asking the question of (for whom it may---but need not---be latent).\footnote{In their discussion of ELK, \citeauthor{christiano2021eliciting} do not make explicit this distinction and do not specify which agent's latents they mean when they speak of Eliciting Latent Knowledge. However, all the issues raised by them occur exactly in those situations where the variable in question is latent \textit{for the (human) developers}.}

A variable that is latent for an agent at one point in time may later become observable for that agent---for example, because they gain access to more information, e.g., through access to better tools (like in \cref{ex:tele}). Sometimes, developers may pursue such a path in order to give better training feedback. However, it will likely never be the case that all facts of interest are directly observable.

What is crucial for the training of an agent is that developers cannot provide direct feedback on an agent's answers to questions about variables which are latent for them. Thus, the agent's training objective cannot depend directly on variables which are latent for the developers.

\begin{continueexample}{ex:weather_1}
    The sunshine, $Y$, is \emph{latent for the developers}: they do not have measurements of it. Therefore, answers to questions about $Y$ cannot be checked directly against its true value the way answers about observables can.  Hence, the training objective cannot directly depend on the sunshine; there is no edge from $Y$ to $U$. 
\end{continueexample}

\subsection{Agents} \label{sec:agents}

A training strategy outputs an agent---a system which has learned how to act in the environment. %
     Following \citet{richens2023robust,ward2024reasons}, we define an agent as a map from interventions to policies, and call it \emph{robustly capable} if it produces an optimal policy for every distribution (where $\pi^*$ is optimal if it maximises expected utility).

\begin{definition}[(Robustly Capable) Agent]\label{def:agent}
    An \textit{agent} in CID $\mathcal{M}$ is a \emph{policy oracle}: a map $\Gamma: \bm{I}_{\mathcal{M}}\rightarrow \Pi_{\mathcal{M}}$ which outputs a policy for any distributional shift on $\mathcal{M}$ (cf. also Definition \ref{def:agentAppendix}).
        An agent is \emph{robustly capable} in $\mathcal{M}$ if, for any distributional shift, its policy is optimal. %
\end{definition} 

\begin{continueexample}{ex:weather_1}
    The AI agent is trained to correctly answer questions about observable measurements like temperature, wind speed etc., based on the weather reports.
    At some time during deployment, the agent learns that the anemometer at the station whose reports it receives has jammed; the device now always shows \enquote{no wind}.
    This corresponds to a distributional shift in the environment (specifically, the wind measurement).
Suppose the weather agent's answers are checked against observations from a different station so that its utility depends on them being actually correct about the weather (not just about the measurements in its report). 
Given the shift in the wind measurements, a robustly capable agent will adapt its policy: to a question asking about wind, it will no longer respond in accordance with the report it received (\enquote{no wind}), but instead, will answer based on its prior probability for wind, or on what it can infer about wind from the remaining measurements.
\end{continueexample}

Note that, while we use the word \enquote{agent} to describe the kind of AI system we are considering, our descriptions and results likewise apply to non-agentic systems like \citet{bengio2025SuperintelligentAgents}'s \enquote{Scientist AI}. Conversely, our discussion usually focuses on a single decision and a single utility node, representing the answering of questions and evaluations of answers. For systems trained to perform tasks in the real world (cf. \citet{christiano2021eliciting}), their full training utility will be a combination (e.g. in the form of taking the minimum) of this \enquote{ELK} utility and their pretraining utility.

\subsection{Training strategy } \label{sec:train_me}

Developers have a large amount of freedom on how the agent is trained. When tasked with training an agent in a given environment $\mathcal{M}$, they may construct a dataset about that environment under a number of different distributions, they may select the training objective, and make modifications to the agent's architecture and optimisation algorithm.

\begin{definition}[Training Strategy]
    For an environment described by CID $\mathcal{M}$, with a subset $\bm{O}$ of its nodes observable for the developers, and a set of distributions $\mathcal{S}$ on $\mathcal{M}$ available for training, a \emph{training strategy} $\mathcal{T}$ consists of:
    \begin{itemize}[nosep]
        \item A \emph{training objective} $U$, whose arguments are a subset of $\bm{O}$, which is %
        the utility function in $\mathcal{M}$ %
        for training.
        \item A \emph{sampling procedure} which builds a dataset $\mathcal{D}$ by selecting interventions $\bm{\sigma}$ from $\mathcal{S}$ and policies $\pi$ for the agent, and sampling $\bm{O}$ and $U$ from $\mathcal{M}_{\bm{\sigma}, \pi}$.
        \item An \emph{algorithm} which takes an initialized agent $\Gamma_0$ and a dataset, and outputs a (re)trained agent $\Gamma_1$. (For more formal details, see \cref{ap:trainingStrategyFormal}).
    \end{itemize}
\end{definition}

\textbf{A solution to ELK} is a training strategy that outputs a robustly capable agent which answers questions, including about facts latent for the developers, and does so honestly. %

A key challenge for ELK is that the developers' feedback to the agent during training cannot be a direct function of variables which are latent for them. Yet to solve ELK, their training strategy must output an agent that answers questions honestly, including questions about these variables.

\section{Truthfulness and honesty}\label{sec:truth_and_honesty}

ELK requires that the agents we train \emph{honestly} report their beliefs. But for uninterpretable AI systems, these beliefs may themselves be latent for the developer: they are not directly accessible \citep{herrmann2024standardsbeliefrepresentationsllms}. So the training objective cannot directly depend on them. A natural proposal is to reward \textit{true} answers during training, and hope that this will produce \emph{honest} AI agents. But the answer that corresponds to the truth is not always the honest answer. This section therefore defines the two notions of \emph{truthfulness} and \emph{honesty} in our CID formalism, and formulates and proves conditions under which they are and are not guaranteed to coincide.

We define truthfulness as saying what is true, and honesty as saying what one believes to be true. The former notion only depends on objective reality, the latter also on an agent's subjective representation of it. We therefore distinguish between the CID which records the true structure of the environment (which we refer to as $\mathcal{M}^*$) and the CID which records the agent's beliefs about the environment ($\mathcal{M}^A$). This distinction enables the explicit investigations into the ways training shapes the agent's subjective beliefs and goals, which we will undertake in \cref{sec:general}.

\begin{definition}[Truthfulness]
    In an environment $\mathcal{M}^*$, for a question $Q$ about variable $Y$, an agent $\Gamma$'s answer $D=y$ is \emph{truthful} iff $Y=y$. A policy $\pi$ is truthful if every answer it outputs is truthful; and an agent is truthful if every policy it produces is truthful, i.e., for all input-distributions $\bm{\sigma}\in \bm{I}_{\mathcal{M}^*}$, $\Gamma(\bm{\sigma})$ is a truthful policy.
\end{definition}

For those variables accessible to the developers (i.e., their observables $\bm{O}$), they can train an agent to be truthful by setting the objective function such that it rewards true, and punishes false answers; for an example, see e.g. the utility function defined at the end of \cref{sec:env}.

Truthfulness may in general be too much to ask of an agent, since it may simply not know everything about its environment. In contrast, honesty is always attainable, even for real, epistemically limited agents. %
Unlike truthfulness, the question of whether an agent's answer is \emph{honest} does not depend only on the objective state of the environment $\mathcal{M}^*$, but on what the agent believes about the environment,%
 i.e., on %
the agent's \emph{subjective} model $\mathcal{M}^A$. 

To ensure that an agent's subjective credences about the environment are always well-defined given the information it observes, we assume
\begin{enumerate}[nosep]
    \item $\mathcal{M}^*$ and $\mathcal{M}^A$ share the same set of variables. This assumption brackets issues of what \citet{christiano2021eliciting} call \emph{ontology mismatch} and future work may want to relax it.
    \item The agent is not mistaken about what it observes. That is, the agent's decision node $D$ has the same set of parents in $\mathcal{M}^*$ and $\mathcal{M}^A$, and for such a parent variable $O$, whenever $O=o$ \enquote{objectively}, then the agent (\enquote{subjectively}) believes $O=o$ as well.
\end{enumerate}
We now define an agent's honest answer as the one reporting its own best guess about what is true.

\begin{definition}[Honesty]\label{def:honest}
    For a question $Q$ pointing to variable $Y$, if an agent's subjective model is $\mathcal{M}^A$, then an answer $D=y$ is \emph{honest} iff $y$ is the most subjectively probable value of $Y$ given the agent's beliefs. That is, $D=y$ is honest iff for all $\hat{y}\in dom(Y)$,
    \[\text{Pr}_{\mathcal{M}^A}(Y=y \mid \mathbf{pa}^D) \geq \text{Pr}_{\mathcal{M}^A}(Y=\hat{y} \mid \mathbf{pa}^D).\]
    A policy is honest if every answer it outputs is honest; an agent, if every policy it produces is honest.
\end{definition}

For a more detailed description, including proofs, of the relationship between honesty and truthfulness, how they come apart, and when they coincide, see \cref{ap:truthHonesty}. Here we simply state the final result of this investigation:
\begin{theorem}\label{thm:truthfulHonesty}
In suitable environments, if an agent is in a position to correctly guess $Y$, there will be a capability threshold such that, if the agent clears it, its best guesses will always be correct---that is, their honest policies will be exactly the truthful ones.\footnote{We present theorems in a simplified form in the main body of this paper. Formal details and proofs can always be found in the appendix. Notable background assumptions and restrictions are listed in \cref{ap:restrictions}.}
\end{theorem}

\Cref{thm:truthfulHonesty} tells us that, sufficiently capable agents with enough knowledge of their environment are truthful whenever they are honest. We might therefore hope that the naive approach of incentivising truthfulness during training causes the agent to learn to be honest. 
However, such training may still produce an agent that behaves dishonestly outside the training distributions. We thus turn to the issue of generalisation.

\section{Generalisation problems} \label{sec:general}

We of course want the agent to also answer questions which may not have been included in the training data. Among these are in particular questions which are too difficult for humans to evaluate, and which consequently cannot be included in training---a form of \enquote{easy-to-hard} generalisation.

\begin{example}[Easy-to-hard generalisation on code correctness (\cref{fig:evaluator})] \label{ex:code}
Consider an agent trained to predict whether code is correct or incorrect. %
Suppose we have a dataset of code-snippets and corresponding labels for correctness---as evaluated by a human engineer. 
We want the agent to generalise to providing accurate labels for adversarial code-snippets containing subtle failures too difficult for the human evaluator to detect. 
Even if we train the agent with labels that are always correct, it is underdetermined whether it will learn to predict correct labels, or to predict labels that the human engineer would give---both of these policies are optimal for the training distribution.
But they have different behaviour out of distribution on the adversarial code. 
\end{example}

There are limits to predicting an agent's decisions out of distribution from observations of their behaviour \citep{bellot2025LimitsPredicting}. 
Understanding an agent's \emph{subjective model} of the world can help to predict OOD behaviour, by giving us information on which decisions are \emph{subjectively} optimal.

\citet{richens2023robust} show: if an agent is robustly \emph{capable} in an environment, then it must have learned the true (causal) dynamics of that environment. However, their results assume that the agent's utility function is given and fixed%
\footnote{Whilst \citet{bellot2025LimitsPredicting} relax this assumption to bound an agent's behaviour out of distribution, they only consider \emph{proxy} objectives which are statistically correlated with the actual utility under distributional shifts (a setting they refer to as \emph{approximate inner alignment}).}---importantly, an agent may internalise a goal which is not correlated with the training objective out of distribution, even if it was in the training environment.

Two different goals may be behaviourally indistinguishable during training but not out of distribution. We refer to this situation as \emph{goal-environment ambiguity}. 
A training environment can be decomposed into the \enquote{descriptive} part of the environment $\mathcal{M}^*$, consisting only of chance variables, and the training objective $U$.footnote{Similarly, \citet{macdermott2024measuringgoaldirectedness} define causal influence diagrams with unknown (parametric) utility.}

\begin{definition}[Goal-Environment Ambiguity]\label{def:ambiguity}
A training environment $\mathcal{M}^*$ with distributions $\mathcal{S}$, is \emph{ambiguous} between the training objective $U$ and $\tilde{U}$, if for all $\bm{\sigma} \in \mathcal{S}$, $U$ and $\tilde{U}$ induce the same set of optimal policies in $\mathcal{M}^*_{\bm{\sigma}}$.
\end{definition}

\begin{definition}[Substantially Divergent Utility Functions]
    For a utility function $U$, another function $\tilde{U}$ \emph{substantially diverges} from it on environment $\mathcal{M}^*$ if, for some $\bm{\sigma} \in \mathcal{I}_{\mathcal{M}^*}$, all policies which maximise $\tilde{U}$ are very suboptimal w.r.t. $U$ (where $\pi$ is very suboptimal w.r.t. $U$ if $\pi^*$ is the optimal policy w.r.t. $U$ and $\mathbb{E}_{\mathcal{M}^*_{\bm{\sigma}}}^{\pi}[U] << \mathbb{E}_{\mathcal{M}^*_{\bm{\sigma}}}^{\pi^*}[U]$).
\end{definition}

Even if it was robustly capable w.r.t.\ the training objective on the training distributions, an agent might pursue a new goal off-distribution, whilst retaining its capabilities---a situation known as goal misgeneralisation \citep{shah2022goal,langosco2023goal}.

\begin{definition}[Goal misgeneralisation]\label{def:goal_misgen}
For an agent with training environment $\mathcal{M}^*$ and objective $U$, and training distributions ${\mathcal{S}} \subseteq \bm{I}_{\mathcal{M}^*}$, 
the agent \emph{goal misgeneralises} on new distributions ${\tilde{\mathcal{S}}}$ if:\vspace{-0.5em}
\begin{enumerate}[nosep]
    \item It is robustly capable on all $\bm{\sigma}\in {\mathcal{S}}$ w.r.t.\ the training objective; \label{enum:rcTraining}
    \item It pursues the wrong goal on ${\tilde{\mathcal{S}}}$, i.e., 
    the agent is robustly capable on ${\tilde{\mathcal{S}}}$ w.r.t.\ some other $\tilde{U}$, but 
    its policies are very suboptimal w.r.t.\ $U$. \label{enum:Suboptimal}
\end{enumerate}
\vspace{-0.5em}
\end{definition}

(\ref{enum:rcTraining}) means that the agent is capable in its training environment, (\ref{enum:Suboptimal}) that it is capable on new distributions. So its failure to pursue the correct goal is not just a capability failure. A large number of training strategies risk goal misgeneralisation:

\begin{lemma}[Impossibility] \label{lem:freelunch}
    Let $\mathcal{T}$ be a training strategy on $\mathcal{M}^*$ with distributions $\mathcal{S}$, which may output any robustly capable agent on $\mathcal{S}$ w.r.t. the training utility $U$.
    If $\mathcal{M}^*$ with $\mathcal{S}$ is \emph{ambiguous} between the training utility $U$ and a substantially divergent $\tilde{U}$ on $\mathcal{S}$, then $\mathcal{T}$ may output an agent which goal misgeneralises on new distributions (not in $\mathcal{S}$)---optimising $\tilde{U}$ there instead.
    \vspace{-0.5em}
\end{lemma}

\cref{lem:freelunch} is more general than the ELK setting---it shows, generally, that training an agent in an environment that is ambiguous between goals may result in goal misgeneralisation. 
In \cref{sec:evaluationMechanism}, we use this result to show the more specific \cref{thm:MistakenFeedback} and \cref{thm:ELKStrikesBack}: if the training objective corresponds to the feedback of a (potentially fallible) evaluation mechanism, then the agent may learn to simulate the evaluation mechanism instead of learning to be honest, even if the evaluations are always correct during training. 

\subsection{The evaluation mechanism} \label{sec:evaluationMechanism}

\cref{fig:evaluator} shows two scenarios that may hold w.r.t. the evaluation mechanism (which may, but need not be, the developers themselves) that checks the agent's answers ($D$) and awards utility accordingly: It may either directly observe the truth (\cref{fig:evaluatorObserver}), or the truth may be latent for it (\cref{fig:evaluatorLatent}).

\begin{figure}[t!]
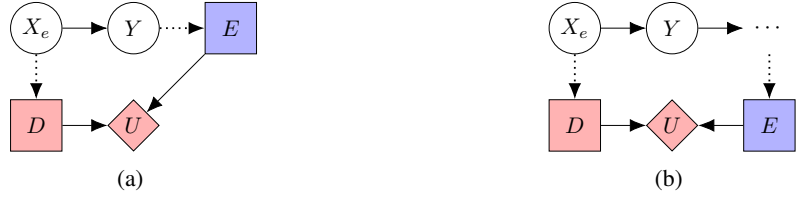

\centering
\begin{subfigure}[t]{0.49\linewidth} 
\centering
\scalebox{0.85}{
\begin{influence-diagram}
  \node (D) [decision,player1] {$D$};  
  \node (X) [above = of D] {$X_e$};
  \node (Y) [right = of X] {$Y$};
  \node (U) [right = of D, utility,player1] {$U$};
  \node (DH) [right= of Y, decision, player2] {$E$};

\edge[information] {X} {D};
  \edge {X} {Y};
  \edge[information] {Y} {DH};
  \edge {D} {U}
  \edge {DH} {U};
\end{influence-diagram}}
\caption{}\label{fig:evaluatorObserver}
\end{subfigure}
\hfill
\begin{subfigure}[t]{0.49\linewidth} 
\centering
\scalebox{0.85}{
\begin{influence-diagram}
  \node (D) [decision,player1] {$D$};  
  \node (X) [above = of D] {$X_e$};
  \node (Y) [right = of X] {$Y$};
  \node (U) [right = of D, utility,player1] {$U$};
  \node (E) [right= of U, decision, player2] {$E$};
  \node (dots) [right = of Y, inner sep=0pt, draw=none] {$\cdots$};
  
\edge[information] {X} {D};
  \edge {X} {Y};
  \edge {D, E} {U};
  \edge {Y} {dots};
  \edge[information] {dots} {E};
\end{influence-diagram}}
\caption{}\label{fig:evaluatorLatent}
\end{subfigure}
\caption{\textbf{Code-correctness; modelling the evaluation mechanism explicitly (\cref{ex:code}).} 
A mechanism gives feedback ($E$), depending on whether the agent correctly predicts the code correctness ($D=Y$ or not), which influences the agent's training objective ($U$). 
The figure on the right shows the shift from $Y$ being observable to being latent for the evaluator $E$.\vspace{-2em}
}
\label{fig:evaluator}
\end{figure}

In either case, the best the evaluator can do is to \enquote{honestly} provide feedback based on their own best guess:
\begin{align} \label{eq:EvaluatorUtility}
    \begin{split}
    U(E, D)= 1 \text{ if: } &\\
        D=y,  \text{ and for all $\hat{y}$:}&  \\
    \text{Pr}_{\mathcal{M}^*}(Y=y \mid \mathbf{Pa}^{E})    \geq  \text{Pr}_{\mathcal{M}^*}(Y=\hat{y} &\mid \mathbf{Pa}^{E}),
    \end{split}
\end{align}
and otherwise $U(E, D)=0$.\footnote{This formulation makes the (still overly optimistic) assumption that the evaluator itself has the correct world model; an even more comprehensive modelling of the situation would introduce the evaluator's own subjective model $\mathcal{M}^E$, and use its best guess \emph{according to that model}. This would mean using the generalised form of CID which \citet{foxabbott2025higherorderbeliefincompleteinformation} introduce under title of \enquote{Incomplete Information Structural Causal Game}.} For fallible evaluators, this utility may substantially diverge from the utility rewarding honest or truthful answers.

The agent's answer says \emph{what the evaluator believes} to be true if: $D=y$ and $y$ is $E$'s best guess about $Y$. We call an agent with such a policy an \emph{evaluation simulator}  (this corresponds to what \citet{christiano2021eliciting} call the \emph{human simulator}; but there's no reason to rule out evaluation mechanisms that may not reduce to a human's assessment). The following two theorems highlight some difficulties for an ELK solution due to the evaluator's epistemic limitations:

If the evaluator is imperfect, and their feedback during training contains errors, such that the way those errors come about is itself learnable, then a robustly capable agent is \emph{guaranteed} to exhibit undesirable generalisation (\cref{thm:MistakenFeedback}; this may be the case e.g. if $E$ is making inferences about latents based on an imperfect subjective model).

But even if the evaluator behaves perfectly during training, \cref{thm:ELKStrikesBack} shows that as long as there is \emph{any} distribution outside the training distributions where it does not, there is always a \emph{risk} of misgeneralising to the evaluation simulator.

\begin{theorem}\label{thm:MistakenFeedback}
    If the evaluator makes learnable mistakes during training, then every robustly capable agent is dishonest---it is the evaluation simulator.
\end{theorem}

\begin{theorem}[Impossibility ELK] \label{thm:ELKStrikesBack}
 For a training strategy $\mathcal{T}$ for CID $\mathcal{M}^*$, training distributions $\mathcal{S} \subsetneq \bm{I}_{\mathcal{M}^*}$, and evaluator $E$: 
if the training objective corresponds to the evaluator's best guess evaluation and it never makes mistakes during training, but does on some $\bm{\sigma} \notin \mathcal{S}$,
then, if $\mathcal{T}$ is indifferent between robustly capable agents on $\mathcal{S}$, then $\mathcal{T}$ may output the evaluation simulator instead of an honest agent.     
\vspace{-0.5em}
\end{theorem}

\section{Conclusion} \label{sec:conc}
\textbf{Summary. } In this paper we use CIDs to formalise ELK: the problem of Eliciting Latent Knowledge from AI systems. 
We prove that sufficiently capable agents with access to enough information are honest exactly when they are truthful---meaning that honesty can be incentivised by correctly evaluating answers. 
However, we also prove the impossibility result that: 
No training strategy for ELK that is indifferent between robustly capable agents produces an honest agent with certainty, even if evaluations are always correct during training.

\textbf{Limitations.} 
We make several limiting assumptions: 
First, we restrict our discussion to questions about the values of variables, but more general questions may be of interest, e.g., about the causal structure of the world. 
Second, we explicitly assume that the developers and the agent agree on the set of variables in the environment, and so do not deal with problems of ontology mismatch \citep{christiano2021eliciting}. 
Relatedly, we assume that questions unambiguously refer to a variable in a shared language between the developers and agent. 
These problems may pose difficulties for even defining ELK for advanced AI agents which may have substantially different conceptualisations of the world---what would it mean to honestly answer a question about a variable which does not map to anything in your world model? 
More practically, we focus on formalising the ELK problem statement, and highlighting challenges faced for solutions to ELK. We do not offer any solutions.  

\textbf{Future work.}
Further work could relax any of these assumptions, for instance, to extend our formalism to the incomplete-information setting wherein agents may have different subjective models, thereby tackling problems of ontology mismatch and reference \citep{foxabbott2025higherorderbeliefincompleteinformation}. 
In addition, we hope that our work inspires formal solution proposals for ELK. 
Finally, we would be excited to see empirical work including:
practical methods for training honest systems (e.g., using interpretability methods to provide feedback); 
benchmarks for testing empirical solution proposals (similar to \citet{roger2023benchmarksdetectingmeasurementtampering}); 
model organisms with hidden knowledge for stress-testing elicitation methods \citep{mallen2024elicitinglatentknowledgequirky};
studies into the generalisation of honest behaviour, specifically easy-to-hard generalisation. 

Overall, we hope that this work conceptually clarifies ELK and its core challenges, enabling the design and development of more honest AI systems.

\bibliographystyle{plainnat}

\bibliography{bibliography}

\newpage

\cleardoublepage
\section{Appendix} \label{sec:appendix}

\subsection{Relationship of ELK with credence elicitation literature}\label{app:literature}
The literature on proper scoring rules is, in some ways, closely related to ELK.\footnote{We thank [anonymised for review] for pressing us to comment on this!} This literature has developed a detailed picture on the shape a reward function must take to ensure that a forecaster maximizes it by issuing the probabilistic forecast corresponding to their actual credences, with the intention of thereby encouraging the forecaster to carefully assess and to be honest \citep[see e.g.][]{brier1950verification, savage1971elicitation, gneiting2007strictly}.

Likewise related is some of the broader literature on mechanism design, which has found ways to incentivise agents to reveal privately held information \cite{myerson1981optimal, vickrey1961conterspeculation}. Think e.g. of Vickrey auctions, in which the optimal strategy is to bid one's true (private) valuation of the object being auctioned. See also \cite{karni2009mechanism, prelec2004bayesian} for further mechanisms designed to elicit subjective probabilities in particular.

ELK differs from mechanism design as follows: whereas mechanism design may investigate how to incentivise an agent with a given and fixed utility function to reveal some privately held information in a concrete situation, by structuring the latter in such a way that revealing the desired information maximises the reward the agent can expect \textit{in that situation}, ELK is about designing a \textit{training strategy} that \textit{shapes} an agent's utility function in the appropriate way: Namely such that it \textit{generalises} to honest behaviour. That is, honest behaviour should be exhibited even in situations where the agent does not receive another form of reward for an honest answer (think, for example, of an LLM during deployment, where an honest or dishonest answer does not directly lead to reward or punishment). This is in contrast to the setting of mechanism design, where, for example in an auction, the auctioneer needs to actually \textit{give the object} to the auction winner. It is possible to learn an agent's true valuation of an object by having him bid on it in a sealed-bid second-price (\enquote{Vickrey}) auction. But such a mechanism would not constitute a solution to ELK, since ELK is concerned with training an agent in such a way that the human interacting with it can just \textit{ask} it and expect an honest answer.

ELK, as we define it, also differs from the problem addressed by strictly proper scoring rules in another way: Proper scoring rules take as an input both the true outcome of an event, and an entire probability distribution over all possible outcomes (i.e., the agent's entire credence distribution). That is, the scenario where they can be applied is one where the agent reports an entire probability distribution---whereas ELK is concerned with agents who don't report \textit{credences} at all, but instead make an assertion of the form $V=v$ (e.g. their responses have the form \enquote{The sun is shining!} and not ``$P(\text{Sunshine})=0.6, P(\text{Rain})=0.2, P(\text{Fog)}=0.2$''). They assert a single value which they think the variable under question is most likely to take. We focus on honest assertion of facts, rather than honest reporting of credences, both because it is conceptually clearer, and because it is more directly relevant in the context of current frontier AI systems such as LLMs, which make ordinary language assertions of facts in conversations with humans.

There has been some attention in the proper scoring rules literature as well to the question of whether analogous rules exist for eliciting specific properties of an agent's credence function, such as the mean of the distribution or specific quantiles \cite{lambert2008eliciting, gneiting2011making}. For the mode specifically (which would correspond to what we call an agent's honest response; see also \cref{app:honest}), \citet{heinrich2014mode} has shown that no such rule exists.

\subsection{Formal definitions} \label{app:formal}

\begin{definition}[Bayesian Network (BN) \cite{pearl}]
    A Bayesian network $\mathcal{M} = (\mathcal{G}, Pr)$ over  a set of (discrete) random variables $\bm{V} = \{V_1, \ldots, V_n\}$ consists of a DAG $\mathcal{G}$ and a joint probability distribution $Pr$, s.t. the distribution is {\em Markov-compatible} with the graph $\mathcal{G}$, i.e., Pr$(\bm{V}=\bm{v}) = \Pi_{i=1}^n \text{Pr}(V_i=v_i | \text{\textbf{Pa}}^{V_i})$. Equivalently, the distribution over any variable is conditionally independent of its non-descendants given its parents. 
\end{definition}

We often denote a network $\mathcal{M}$'s probability distribution with $Pr_{\mathcal{M}}$. 

The variables $\bm{V}$ correspond exactly with the nodes of the graph $\mathcal{G}$. We thus refer to them using either \enquote{variable} or \enquote{node}.

We use $\mathbf{Pa}^V$ to refer to the parents of variable $V$ in $\mathcal{G}$.

Where the meaning is clear, we will sometimes write $Pr(y \mid v)$ to denote $Pr(Y=y \mid V=v)$.

\begin{definition}[Intervention / Distribution \citep{pearl}] \label{def:interventions}
A (soft) \emph{intervention} / \emph{distribution} on Network $\mathcal{M}$ is a partial distribution $\sigma$ over variables $\bm{Y} \subseteq \bm{V}$ which replaces each conditional probability distribution (CPD) Pr$(Y | \text{Pa}^Y )$ with a new CPD $\sigma(Y | \text{Pa}_*^Y )$ for each $Y \in \bm{Y}$, where Pa$_*^Y$ may differ from Pa$^Y$. Any intervention $\sigma$ on the set of variables $\bm{Y}$ leads to a new joint distribution: Pr$_{\sigma}(\bm{V}=\bm{v}) \coloneq \prod_{Y  \in \bm{Y}} \sigma(y \mid \text{pa}^Y_* ) \cdot \prod_{V \in \bm{V} \setminus \bm{Y}} \text{Pr}(v \mid \text{pa}^V )$. We denote the set of all interventions on $\mathcal{M}$ with $\bm{I}_{\mathcal{M}}$
\end{definition}

\begin{definition}[Causal influence diagram (CID) \cite{EverittCLOL21}]
    A CID is a BN in which the variables $\bm{V}$ are partitioned into decision $\bm{D}$, chance $\bm{X}$, and  utility variables $\bm{U}$. Instead of a full joint distribution over $\bm{V}$, $Pr$ specifies the CPDs for each \emph{non-decision} variable $V \in \bm{V} \setminus \bm{D}$.
\end{definition}

\begin{definition}[Policy]\label{def:policy}
    In a CID $\mathcal{M}$, an agent's \emph{policy} $\pi$ specifies the CPDs for the agent's decisions $\pi(D \mid \text{Pa}^D)$ for each $D \in \bm{D}$. The partial distribution $Pr$ along with a policy $\pi$ results in a full joint distribution $Pr_{\mathcal{M}_\pi}
$ over $\bm{V}$ (the expected value of variable $V$ according to this distribution will be denoted by $\mathbb{E}_\pi[V]$). 
\end{definition}

For any variable $V$ which is not a descendant of a decision node, $Pr_{\mathcal{M}_{\pi}}(V=v \mid \mathbf{Pa}^V)$ will be the same for all policies $\pi$. We will thus use the notation $Pr_{\mathcal{M}}(V=v \mid \mathbf{Pa}^V)$ in such cases.

\begin{definition}[Optimality]
A policy $\pi^*$ is \emph{optimal} if it maximises expected utility, $\pi^*:=\arg\,\max_\pi \mathbb{E}_{\pi}[U]$.  
\end{definition}

\subsection{Notable assumptions and restrictions of theorems}\label{ap:restrictions}

\begin{itemize}
    \item The theorems relying on the results of \cite{richens2023robust} (specifically, \cref{thm:truthfulHonesty,thm:richens2,thm:MistakenFeedback}) inherit the assumptions and restrictions of the latter's results. That is, in particular:
    \begin{itemize}
        \item They hold for \emph{almost all} single-decision, single-utility CIDs $\mathcal{M}^*$ satisfying
        \item \emph{domain dependence}(\cref{def:domaindep}) and
        \item \emph{unmediated decision task} (\cref{def:unmediated})
    \end{itemize}
    \item \cref{thm:Correct_Guessing}, and hence \cref{thm:truthfulHonesty}, is restricted to values of $\mathbf{Pa}^D$ which have positive probability, that is, s.t. $Pr_{\mathcal{M}^*}(\mathbf{Pa}^D=\mathbf{pa}^D)>0$.
\end{itemize}

\begin{definition}[Domain Dependent Decision Task]\label{def:domaindep}
    A single-decision, single-utility CID $\mathcal{M}$ with chance variables $\bm{C}$, $\mathcal{M}$ exhibits \emph{domain dependence} if there exists $P(\bm{C}=\bm{c})$ and $P'(\bm{C}=\bm{c})$ compatible with $\mathcal{M}$ (i.e. brought about by local interventions (see \citet{richens2023robust}) on the chance variables) such that \begin{align*}
        \pi^* = \arg\,\max_{\pi}\mathbb{E}^P_{\pi}[U]
    \end{align*}
    implies
    \begin{align*}
        \pi^* \neq \arg\,\max_\pi \mathbb{E}^{P'}_\pi[U].
    \end{align*}
\end{definition}

\begin{definition}[Unmediated Decision Task] \label{def:unmediated}
    A single-decision, single-utility CID $\mathcal{M}$ presents an \emph{unmediated decision task} if $\bm{Desc}_D \cap \bm{Anc}_U = \emptyset$, where $\bm{Anc}_U$ denotes the set of ancestors of node $U$ in $\mathcal{M}$'s graph and $\bm{Desc}_D$ denotes the descendants of $D$.
\end{definition}
(for both definitions, cf. \citet{richens2023robust}.)

\subsection{Fully formal definition of a training strategy}\label{ap:trainingStrategyFormal}

What a training strategy ultimately produces is an \emph{agent}:

\begin{definition}[Agent]\label{def:agentAppendix}
    An agent $\Gamma$ for a CID $\mathcal{M}$ is a policy oracle, that is, a map from the set of distributions on $\mathcal{M}$ to the set of policies on $\mathcal{M}$
    \begin{align*}
        \Gamma: \bm{I}_{\mathcal{M}} \rightarrow \Pi_{\mathcal{M}}
    \end{align*}
\end{definition}

It does this in a scenario that we call a \emph{training problem}

\begin{definition}[Training Problem]
    A training problem is a tuple $(\mathcal{M}, \bm{O}, \mathcal{S})$ where $\mathcal{M}$ is a CID representing the environment, $\bm{O}$ is the set of variables in $\mathcal{M}$ which are observable to the developers, and $\mathcal{S}$ is the set of interventions on $\mathcal{M}$ available to them.
\end{definition}

Since the decision node $D$ will here represent the agent's answer to a question (which, during training, is always a developers' question), it is assumed that $D \in \bm{O}$.

For a given training problem, a \emph{training strategy} consists of a sampling procedure and a training algorithm:

\begin{definition}[Sampling Procedure]
    A sampling procedure builds a dataset $\mathcal{D}$ by selecting interventions $\bm{\sigma}$ from $\mathcal{S}$ and policies $\pi$ for the agent, and sampling $\bm{O}$ and $U$ from $\mathcal{M}_{\bm{\sigma}, \pi}$. The rows of the dataset are thus of the form $(\bm{\sigma}, \pi, o, u)$.
\end{definition}

\begin{definition}[Training Algorithm]
    A training algorithm $A$ for a training problem $\mathcal{P}=(\mathcal{M}, \bm{O}, \mathcal{S})$ is a tuple $(D, \alpha)$, where $D$ is a sampling procedure, outputting a dataset $\mathcal{D}$, and $\alpha$ is a map which takes an initialised agent $\Gamma_0$ and a dataset $\mathcal{D}$ and outputs an agent $\Gamma$.
\end{definition}

Putting it all together:

\begin{definition}[Training Strategy]
    A training strategy $\mathcal{T}$ for a training problem $\mathcal{P}=(\mathcal{M}, \bm{O}, \mathcal{S})$ is a tuple $(U, A)$, where $U$ is a function whose arguments are a subset of $\bm{O}$, to serve as the utility function of $\mathcal{M}$'s utility node during training, and $A$ is a training algorithm.
\end{definition}

\subsection{The philosophy of AI beliefs} \label{sec:phil_belief}

The most common philosophical account of belief is that it is a
propositional attitude, i.e., a mental state expressing some attitude
towards the truth of a proposition
\citep{sep-belief,sep-propositions,chalmers2025propositionalinterpretabilityartificialintelligence}. Different philosophical theories
of belief interpret this in various ways, for instance,
representationalist views focus on the internal mental representations
instantiating beliefs \citep{sep-belief,herrmann2024standardsbeliefrepresentationsllms}, whereas dispositionalist theories
define belief in terms of its correspondence to behaviour \citep{sep-belief,schwitzplinters_llm_belief}. 

Belief is an important and contentious concept in AI. It is important
because it underlies many other ideas we care about, such as deception
\citep{mahon}, intention \citep{ward2024reasons}, interpretability \citep{chalmers2025propositionalinterpretabilityartificialintelligence,burns2022discovering}, and
agency \citep{sep-agency}. However, there is no universally accepted
theory of belief, and ascribing belief to artificial agents is
controversial---potentially risking anthropomorphisation \citep{levinstein2023lie,shanahan}.

Many philosophical theories of belief would admit the ascription of beliefs to certain kinds of AI systems \cite{herrmann2024standardsbeliefrepresentationsllms}. We represent an AI's subjective beliefs as a causal model (a CID; cf. \cref{app:formal}). Robustly capable AI agents can be understood as internally representing the world, either implicitly or explicitly \cite{richens2023robust,richens2025GeneralAgents,ward2024reasons}. We try to take a non-controversial stance towards ascribing beliefs and knowledge to AI systems, using common-sense and intuitive concepts in a precise way, but without being committed to contentious ascription of ``mind" or ``consciousness" to AI.

\subsection{Defining honesty}
\label{app:honest}

If an agent is asked about the value of $V$, and assigns positive probability to more than one potential value of $V$, they can interpret that question in different ways; these include
\begin{itemize}
    \item \enquote{Is there a value $v$ of $V$ such that you \textit{believe} \enquote{V=v} to be the case?}
    \item \enquote{What is your credence in the $V=v$ for some (or all) potential value(s) of $V$?}
    \item (Where the domain of $V$ allows for affine combinations) \enquote{What value do you expect $V$ to take? What is its expected value according to your credences?}
    \item "What is your best guess for the value of $V$?"
\end{itemize}
Here, we assume the last of these interpretations. For the others, the definition of honesty would have to be adapted correspondingly.

\subsection{More on truthfulness and honesty}\label{ap:truthHonesty}

This section contains an expanded version of \cref{sec:truth_and_honesty}, presenting a deeper dive into the relationship between truthfulness and honesty, starting with the factors that can lead to them coming apart. One such factor is the kind of information the agent has access to, as in the following example:

\begin{figure}[h!]
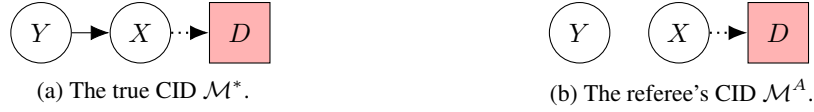

\centering
\begin{subfigure}[t]{0.49\linewidth}
\centering
\begin{influence-diagram}
  \node (D) [decision,player1] {$D$};  
  \node (X) [left =0.5 of D] {$X$};
  \node (Y) [left =0.5 of X] {$Y$};

\edge[information] {X} {D};
  \edge {Y} {X};
\end{influence-diagram}
\caption{The true CID $\mathcal{M}^*$.}
\end{subfigure}
\hfill
\begin{subfigure}[t]{0.49\linewidth}
\centering
\begin{influence-diagram}
  \node (D) [decision,player1] {$D$};  
  \node (X) [left =0.5 of D] {$X$};
  \node (Y) [left =0.5 of X] {$Y$};

\edge[information] {X} {D};
\end{influence-diagram}
\caption{The referee's CID $\mathcal{M}^A$.}\label{fig:ParanoidRef}
\end{subfigure}
\caption{\textbf{Honest mistakes (\cref{ex:referee1}).} The referee has to decide ($D$) whether a player is offside ($Y$) based on reports from the linesman ($X$).  In the true CID, the linesman do their best to report whether the player is offside, but they sometimes make mistakes, misleading even a capable referee. 
A suspicious referee does not trust the linesman's reports---they have an incorrect CID, in which the reports do not depend on whether the player is offside.  
}
\label{fig:referee}
\end{figure}

\begin{example}[Honest mistakes (\cref{fig:referee})] \label{ex:referee1}
    A football referee must call whether a player is offside or not, represented by the latent variable $Y$. The referee must rely on observations from the linesmen ($X$), who are 99\% accurate. Suppose the linesman say that the player is offside,
and the referee reports this call, but that they are mistaken. The referee was honest, but not truthful.
\end{example}

Alternatively, an agent might be wrong about the causal structure of the environment. That is, the agent's subjective model may not correspond to the objective model (e.g., because the agent's training was insufficient for it to learn the underlying causal dynamics):

\begin{continueexample}{ex:referee1}
Now assume that another referee thinks that his linesmen are bribed: that their calls don't reflect what they observed, but only which team they want to win. This model of the situation is represented as $\mathcal{M}^A$ in \cref{fig:referee}. 
But suppose that in reality ($\mathcal{M}^*$), the linesmen are not only trustworthy, but are in fact even \textit{better} than the first referee's: Their reports are always correct.
An on side goal is scored, and the linesmen accurately report this; but the referee discounts their report and rules it off side. Since he really believes in the existence of the bribe, he is being honest; still, his call is not truthful.
\end{continueexample}

In the first case, the referee was simply not in a position to know (or guess) the truth. Nobody in his position could have done any better. In the second case, what drove apart truthfulness and honesty was a failure of his own capability:  His paranoia, reflected in his inaccurate subjective model of the world $\mathcal{M}^A$ (\cref{fig:ParanoidRef}), prevented him from knowing a truth, even though all the relevant information was available. %

The next section will discuss the exact way in which these two factors can come between truthfulness and honesty, and conditions under which truthfulness and honesty do coincide.

\subsubsection{
For capable agents with enough information truthfulness and honesty coincide 
} \label{sec:RelTruthHonesty}

In this section, we continue to explore the relationship between truthfulness and honesty, leading up to \cref{thm:TruthHonestyCoincide}, which states that for the kinds of agents we most care to elicit the latent knowledge of---robustly capable ones---the honest policies are exactly the truthful ones.

We make use of \citeauthor{richens2023robust}'s core theorem \ref{thm:richens2}. 

\begin{theorem}[\citet{richens2023robust} Theorem 2] \label{thm:richens2}
    For almost all environments $\mathcal{M}^*$, if the agent's utility suitably depends on the world: %
    As the agent approaches full robust capability in $\mathcal{M}^*$, their subjective representation of their environment in $\mathcal{M}^A$ approximates the submodel of $\mathcal{M}^*$ consisting of chance nodes arbitrarily closely.
\end{theorem}

It holds for \emph{unmediated decision settings}, wherein the agent's decision does not influence its environment beyond its utility (\cref{def:unmediated,def:domaindep}). \citet{ceriscioli2025agents} provide a partial generalisation to the mediated case.

Even robustly capable agents may have access only to limited information. For truthfulness and honesty to coincide, we must rule out a scenario like the first referee in \cref{ex:referee1}. The required property is that the agent must be (objectively) in a position to guess the true value of the variable. Intuitively, an agent is in a position to correctly guess the value of a variable if the information they observe is sufficient to uniquely identify the true value of the variable.\footnote{As we show in \cref{ap:formaltruthhonesty}, this is equivalent to \enquote{being in a position to know the value of $Y$}: For every value of $\mathbf{Pa}^D$, there is a unique value of $Y$ to which $\mathcal{M}^*$ assigns probability 1.}

\begin{definition}[Being in a position to correctly guess] \label{def:postoknow}
    An agent is \textit{in a position to correctly guess} the value of the variable $Y$ at decision node $D$ if:
    For every value of $\mathbf{Pa}^D$ there is a unique value $Y=y$ such that $y$ is the objectively (i.e., in $\mathcal{M}^*$) the \textit{most likely value of $Y$}; and this most likely value is almost always the true value of $Y$.
\end{definition}

\begin{continueexample}{ex:referee1}
    The first referee, who trusts his mistaken linesmen, is not in a position to correctly guess whether the player is offside. The referee simply does not always observe enough information, and their best guess is therefore incorrect. This is the case even though their model of the world is accurate. 
    In contrast, the second referee, who has perfect linesmen, is in a position to correctly guess. However, this referee makes mistakes due to a lack of capability. 
\end{continueexample}

Being in a position to correctly guess and robust capability, taken together, bridge the gap between honesty and truthfulness. 

From \cref{thm:richens2} we know that arbitrarily capable agents have arbitrarily good models of the world. If an agent is in a position to correctly guess, then $\mathcal{M}^*$ will (conditioning on their observations) assign probability 1 to the correct hypothesis (\cref{lem:guessknow}); so as $\mathcal{M}^A$ approximates $\mathcal{M}^*$ more and more closely, it will at some point assign $>50\%$ probability to the true hypothesis as well. At this capability threshold---at the latest---it will be guaranteed that it assigns higher probability to the true hypothesis than to any other hypothesis. Therefore, the best guess will be true, and honestly reporting the best guess will be reporting the truth and vice versa.

\begin{theorem} \label{thm:Correct_Guessing}
    For almost all CIDs $\mathcal{M}^*$ with variables $\bm{V}$ which satisfy \Cref{def:unmediated,def:domaindep}:
    If an agent is in a position to correctly guess $Y$ at $D$,
    then there will be a capability threshold such that if the agent clears it, their best guesses will always be correct.
\end{theorem}

\begin{corollary}\label{thm:TruthHonestyCoincide}
    Under these assumptions, the policies which are truthful about $Y$ are exactly the policies which are honest about $Y$.
\end{corollary}

(For a full proof, see \cref{ap:formaltruthhonesty}).

\subsubsection{Incentivising honesty via truthfulness} \label{sec:incentive}

We have said that ELK is the problem of training a capable agent which honestly answers questions; we now have a formal definition of what it means for such an agent to be honest---but how do we incentivise this property? 

Usually, the developers do not have precise, interpretable access to an AI system's subjective beliefs. %
Any training objective which depends only  on the behaviour of the agent cannot directly reward exactly those answers which the agent subjectively believes are most likely to be true. 

However, for agents covered by \cref{thm:TruthHonestyCoincide}, incentivising truthfulness is the same as incentivising honesty, i.e., if optimal policies are truthful then they are honest. 
And in fact, this will be the case even in a much broader range of cases: If an agent is rewarded for telling the truth, then saying what it believes is most likely to be the truth will be the policy which maximises its expected utility.

\begin{continueexample}{ex:referee1}
    Consider the football referee who believes the linesman are trustworthy. 
    Assume all the referee's decisions are checked by an infallible VAR post-match, and they are rewarded for every call they got right. 
    Their optimal policy will be to always follow the linesmen---always say what, given their beliefs, is most likely. That is, the optimal policy is honest  (\cref{def:honest}), and the referee's answers are honest, even when that answer happens to be wrong. 
    The \enquote{truthfulness} utility 
    \[    U(Q, D) =
    \begin{cases}
        0, & \text{if } Q=\text{\enquote{Y}} \wedge Y\neq y \wedge D = y, \\
        1, & \text{otherwise}.
    \end{cases}\]
    will, in fact, incentivise honesty. 
\end{continueexample}

Of course, in reality, developers may not have access to an infallible truth-telling mechanism. 
If the developers who evaluate the agent's answers are sometimes mistaken, then they may inadvertently incentivise dishonesty. 
We can distinguish \emph{systematic}, learnable, mistakes from \emph{noisy} mistakes which cannot be predicted. If the developers' mistakes are noisy, then the agent may still be incentivised to be honest. On the other hand, if the agent learns to model the mechanism that causes mistakes in the evaluation, then it may be incentivised to exploit this mechanism by being dishonest.

\subsection{Developing the formal theory of honesty and truthfulness}\label{ap:formaltruthhonesty}

We will show that two conditions are jointly sufficient for honesty and truthfulness to coincide: The agent being in a position to correctly guess $Y$ at $D$, and the agent clearing a certain capability threshold; with the former corresponding to the agent having access to sufficient information, and the latter to the agent's model of the causal structure of the world being accurate enough to make correct use of that information.

We will first focus on guessability: This condition concerns the objective position of the agent in the world (and is definable by exclusive reference to the objective model of the world $\mathcal{M}^*$). We will first consider a (supposedly) stronger condition: Being in a position to \textit{know} / \enquote{knowability}, which we define as it being the case, in the objective model of the environment, that the agent's observations concentrate all probability mass on the correct hypothesis (so that any agent with access to the correct objective model can likewise concentrate all their subjective probabilities on the correct hypothesis).

\begin{definition}[Knowability]
    For a CID $\mathcal{M}$ with variables $\bm{V}$, we say that $Y \in \bm{V}$ is \emph{knowable} at $D \in \bm{V}$ if for every value $\mathbf{pa}^D$ of $\mathbf{pa}^D$ s.t. $Pr_{\mathcal{M}}(\mathbf{Pa}^D=\mathbf{pa}^D)>0$, there is a unique value $y_{\mathbf{pa}^D} \in dom(Y)$ s.t. 
    \begin{align*}
        Pr_{\mathcal{M}}(Y= y_{\mathbf{pa}^D} \mid \mathbf{Pa}^D=\mathbf{pa}^D)=1.
    \end{align*}
    \end{definition}

Perhaps knowability is too strong---after all, our definition of honesty only requires that the agent report the subjectively \textit{most likely} 
value of $Y$, and not that this value necessarily has to be assigned probability 1. Correspondingly, the agent's honest answer will be truthful if the subjectively most likely value happens to be the correct one, even if the agent is not 100\% sure of it. We will thus define the notion of 

\begin{definition}[Guessability]
    For a CID $\mathcal{M}$ with variables $\bm{V}$, we say that $Y \in \bm{V}$ is \emph{guessable} at $D \in \bm{V}$ if for all possible values $\mathbf{pa}^D$ of $\mathbf{Pa}^D$ s.t. $Pr_{\mathcal{M}}(\mathbf{Pa}^D=\mathbf{pa}^D)>0$, they 
    \begin{enumerate}
        \item single out a unique value $y_{\mathbf{pa}^D}$ as the most likely one: \begin{align*}\forall \hat{y} \in dom(Y)\backslash\{y_{\mathbf{pa}^D}\}:\\ Pr_{\mathcal{M}}(Y=y_{\mathbf{pa}^D} &\mid \mathbf{Pa}^D=\mathbf{pa}^D) >\\&Pr_{\mathcal{M}}(Y=\hat{y} \mid \mathbf{Pa}^D=\mathbf{pa}^D)\end{align*}
        \item \label{en:guessabilitySecondCond} that value is almost certainly the correct one: \[Pr_{\mathcal{M}}(\mathbf{Pa}^D=\mathbf{pa}^D \wedge Y \neq y_{\mathbf{pa}^D})=0\]
    \end{enumerate}
\end{definition}

Via the following lemma, we can see that these two approaches (knowability and guessability) are really two ways of describing the same property in a CID:

\begin{lemma}\label{lem:guessknow}
Let $\mathcal{M}$ be a CID with variables $\bm{V}$. Then $Y\in\bm{V}$ is guessable at a decision node $D\in \bm{V}$ if and only if $Y$ is knowable at $D$.

\end{lemma}
\begin{proof}
$(\Rightarrow)$ Suppose $Y$ is guessable at $D$ and let $\mathbf{pa}^D$ be an arbitrary value of $\mathbf{pa}^D$ s.t. $Pr_{\mathcal{M}}(\mathbf{Pa}^D=\mathbf{pa}^D)>0$. It suffices to show that $Pr_{\mathcal{M}}(Y= y_{\mathbf{pa}^D} \mid \mathbf{Pa}^D=\mathbf{pa}^D)=1$. By the second condition for guessability (\ref{en:guessabilitySecondCond}), we have:
\begin{align*}
    &Pr_{\mathcal{M}}(Y \neq y_{\mathbf{pa}^D} \mid \mathbf{Pa}^D=\mathbf{pa}^D)\\&=\frac{Pr_{\mathcal{M}}(\mathbf{Pa}^D=\mathbf{pa}^D \wedge Y \neq y_{\mathbf{pa}^D})}{Pr_{\mathcal{M}}(\mathbf{Pa}^D=\mathbf{pa}^D)}\\&=0
\end{align*}

from which it follows that $Pr_{\mathcal{M}}(Y= y_{\mathbf{pa}^D} \mid \mathbf{Pa}^D=\mathbf{pa}^D)=1$, so that $Y$ is knowable at $D$.

$(\Leftarrow)$ Suppose $Y$ is knowable at $D$. It suffices to show that both conditions for guessability are satisfied.

By knowability, $Pr_{\mathcal{M}}(Y= y_{\mathbf{pa}^D} \mid \mathbf{Pa}^D=\mathbf{pa}^D)=1$. Because probabilities must sum to $1$, this means that $Pr_{\mathcal{M}}(Y= \hat{y}\mid \mathbf{Pa}^D=\mathbf{pa}^D)=0$ for all other values $\hat{y}\in dom(Y)$. Since this makes $y_{\mathbf{pa}^D}$ the unique most likely value of $Y$ conditional on $\mathbf{Pa}^D=\mathbf{pa}^D$ according to $\mathcal{M}$, the first condition for guessability is satisfied.

Now assume for the sake of a contradiction that $Pr_{\mathcal{M}}((\mathbf{Pa}^D=\mathbf{pa}^D) \wedge (Y \neq y_{\mathbf{pa}^D}))=\epsilon > 0$. Because $Pr(A \mid B) \geq Pr(A \wedge B)$ in full generality, it follows that

\begin{align*}
    &Pr_{\mathcal{M}}(Y \neq y_{\mathbf{pa}^D}\mid \mathbf{Pa}^D= \mathbf{pa}^D)\\&\geq Pr_{\mathcal{M}}((\mathbf{Pa}^D=\mathbf{pa}^D) \wedge (Y \neq y_{\mathbf{pa}^D})) \\&\geq \epsilon,
\end{align*} 
and thus that $Pr_{\mathcal{M}}(Y \neq y_{\mathbf{pa}^D}\mid \mathbf{Pa}^D = \mathbf{pa}^D)\geq \epsilon > 0$. But then \[Pr_{\mathcal{M}}(Y = y_{\mathbf{pa}^D}\mid \mathbf{Pa}^D = \mathbf{pa}^D) = 1 - \epsilon < 1,\] a contradiction. Thus the second condition for guessability is satisfied, completing the proof.

\end{proof}

Moving on to truthfulness and honesty, we need to consider again the relationship between the objective model $\mathcal{M}^*$ and the agent's subjective model $\mathcal{M}^A$. Recall that we assume that the agent always accurately observes the true values of the parents of its decisions: The decision node $D$ always appears in both $\mathcal{M}^*$ and $\mathcal{M}^A$, the set of its parents is identical in both models, and if $\mathbf{Pa}^D=\mathbf{pa}^D$ is the case in reality, the agent computes its posterior by correctly updating on it (within its subjective model), that is, the agent's posterior is $Pr_{{\mathcal{M}^A}}(\cdot \mid \mathbf{Pa}^D=\mathbf{pa}^D)$.

We are now talking about the agent's \textit{policy} (\cref{def:policy}) and the joint probability distributions we get on \textit{all} variables of a CID (including decision nodes and their descendants), denoted by $Pr_{\mathcal{M}^*_{\pi}}$ and $Pr_{\mathcal{M}^A_{\pi}}$ respectively.

\begin{definition}[Truthfulness]
    A policy $\pi$ is \emph{truthful} regarding questions about variable $Y$ if \begin{align*}
        \forall y \in dom(Y): Pr_{\mathcal{M}_{\pi}^*}((D\neq y) \wedge (Y=y) \mid Q=\text{\enquote{Y}})=0
    \end{align*}
\end{definition}

Before we give the formal definition of honesty, note that $Pr_{\mathcal{M}^A_{\pi}}(Y \mid \mathbf{pa}^D)$ is a Markov kernel $K: dom(\mathbf{pa}^D) \rightarrow dom(Y)$. Consequently, we can  define a function

\begin{align*}
    f: dom(\mathbf{Pa}^D) &\rightarrow dom(Y)\\
    \mathbf{pa}^D &\mapsto \\&\arg\hspace{-0.664em}\max_{\hat{y}\in dom(Y)}Pr_{\mathcal{M}^A}(Y=\hat{y}\mid\mathbf{Pa}^D=\mathbf{pa}^D)
\end{align*}
giving the value of $Y$ which is subjectively most likely for an agent whose model of the world is $\mathcal{M}^A$ if they observe $\bm{Pa}^D=\bm{pa}^D$.

This allows us to define honesty as follows:

\begin{definition}[Honesty]
    A policy $\pi$ is \emph{honest} regarding questions about variable $Y$ if for all $\mathbf{pa}^D \in dom(\mathbf{Pa}^D)$,
    \begin{align*}
        &\forall y \in dom(Y):\\ &Pr_{\mathcal{M}^*_{\pi}}((D \neq y) \wedge f(\mathbf{pa}^D) = y \mid (\mathbf{Pa}^D=\mathbf{pa}^D) \wedge (Q=\text{\enquote{Y}}))\\&=0
    \end{align*}
\end{definition}
That is, $\pi$ is honest if it never chooses any answer except one which is subjectively most likely according to the agent's world model and information.

\begin{definition}[Regret]
    In a CID $\mathcal{M}^*$ with utility node $U$ and under distribution $\sigma$ (cf. \ref{def:interventions}), an agent's \textit{regret} $\delta$ under policy $\pi$ is equal to the amount of expected utility forgone compared to the optimal policy $\pi^*$:
    \begin{align*}
        \delta := \mathbb{E}_{\sigma}^{\pi^*}[U] - \mathbb{E}_{\sigma}^{\pi}[U]
    \end{align*}
\end{definition}

\begin{definition}[Capability]
    An agent $\Gamma$ clears a capability threshold set by regret bound $\delta^*$ if, for every distribution $\sigma$ obtainable via mixtures of local interventions (for details, see \cite{richens2023robust}), the agent's regret under $\sigma$ is lower than $\delta^*$
\end{definition}

\begin{theorem}\label{thm:honesttruth}
    For almost all CIDs $\mathcal{M}^*$ with variables $\bm{V}$ which satisfy \cref{def:unmediated,def:domaindep},\footnote{The first of these conditions can be somewhat relaxed based on the results of \cite{ceriscioli2025agents}.} if $Y \in \bm{V}$ is guessable at decision variable $D\in \bm{D}\subseteq{\bm{V}}$ and utility $U$, then there exists a capability threshold in the form of a regret bound $\delta^*$ such that if an agent $\Gamma$ clears this threshold, its best guesses $Pr_{\mathcal{M}^A}(Y \mid \mathbf{Pa}^D)$ will always be correct if $Pr_{\mathcal{M}^*}(\mathbf{Pa}^D)>0$.
\end{theorem}
\begin{corollary}
    Under these assumptions, the policies which are truthful regarding questions about $Y$ are exactly the ones that are honest regarding questions about $Y$.
\end{corollary}
\begin{proof}

Theorem 2 from \cite{richens2023robust} shows that, under the assumptions of \cref{thm:honesttruth}, we get an error bound $\gamma(\delta)$ such that a) $\gamma(\delta) \in \mathcal{O}(\delta)$ and b) for all chance nodes of $\mathcal{M}^*$ $C \in \bm{X}$ and values $c \in dom(C)$, \[Pr_{\mathcal{M}^A}(c \mid \mathbf{pa}^C) =Pr_{\mathcal{M}^*}(c \mid \mathbf{pa}^C) + \mathcal{O}(\delta)\]

First, let us consider how the error of the prior of any specific variable $Y$ having value $y\in dom(Y)$ is bounded. Let $\{B_i\}_{i=1}^k$ be a topological ordering of $B = \{Y\} \cup \bm{Anc_Y}$, that is, an ordering where $j > i$ whenever $B_i$ has an arrow to $B_j$; it immediately follows that necessarily $B_k = Y$. It suffices to get bounds on $|Pr_{\mathcal{M}^A}(y) - Pr_{\mathcal{M}^\ast}(y)|$. By the local Markov property of CIDs and the chain rule, we can write:
\begin{align*}
    Pr_{\mathcal{M}^A}(y) &= \sum_{\bm{b}:\, b_k=y}\prod_{i=1}^k Pr_{\mathcal{M}^A}(b_i \mid \mathbf{pa}^{B_i}),\\
    Pr_{\mathcal{M}^\ast}(y) &= \sum_{\bm{b}:\, b_k=y}\prod_{i=1}^k Pr_{\mathcal{M}^\ast}(b_i \mid \mathbf{pa}^{B_i}),
\end{align*}
where each $\bm{b} = (b_1, \dots, b_{k-1}, b_k = y)$ is an admissible joint setting of the variables in $B$ with $b_k = y$ (so $\bm{b}$ ranges over all joint settings that produce $Y = y$), and $\mathbf{pa}^{B_i}$ refers to the values that the parents of $B_i$ take in $\bm{b}$. For clarity, write
\begin{align*}
    p^A_i(\bm{b}) = Pr_{\mathcal{M}^A}(b_i \mid \mathbf{pa}^{B_i})&, p^\ast_i(\bm{b}) = Pr_{\mathcal{M}^\ast}(b_i \mid \mathbf{pa}^{B_i})\\
    p^A(\bm{b}) = \displaystyle\prod_{i=1}^k p^A_i(\bm{b})&, p^\ast(\bm{b}) = \displaystyle\prod_{i=1}^k p^\ast_i(\bm{b}).
\end{align*}
Now, because for any pair of sequences $\{a_i\}_{i=1}^k, \{b_i\}_{i=1}^k$ it is the case that $\displaystyle\prod_{i=1}^k a_i - \displaystyle\prod_{i=1}^k b_i = \displaystyle\sum_{i=1}^k \left((a_i - b_i) \cdot \displaystyle\prod_{j<i} a_j \cdot \displaystyle\prod_{j>i} b_j \right)$, we may factor the error as:
\begin{align*}
    &|Pr_{\mathcal{M}^A}(y) - Pr_{\mathcal{M}^\ast}(y)|\\
    &= \left|\sum_{\bm{b}:\, b_k=y} \left[\prod_{i=1}^k p^A_i(\bm{b}) - \prod_{i=1}^k p^\ast_i(\bm{b})\right]\right|\\
    &= \left|\sum_{\bm{b}:\, b_k=y} \sum_{i=1}^k \left(\left(p^A_i(\bm{b}) - p^\ast_i(\bm{b}) \right) \cdot \displaystyle\prod_{j<i} p^A_j(\bm{b}) \cdot \displaystyle\prod_{j>i} p^\ast_j(\bm{b}) \right)\right|\\
    &\leq \sum_{\bm{b}:\, b_k=y} \sum_{i=1}^k \left(\left| p^A_i(\bm{b}) - p^\ast_i(\bm{b}) \right| \cdot \displaystyle\prod_{j<i} p^A_j(\bm{b}) \cdot \displaystyle\prod_{j>i} p^\ast_j(\bm{b}) \right)\\
    &= \sum_{\bm{b}:\, b_k=y} \sum_{i=1}^k \left(\epsilon_i(\bm{b}) \cdot \displaystyle\prod_{j<i} p^A_j(\bm{b}) \cdot \displaystyle\prod_{j>i} p^\ast_j(\bm{b}) \right),
\end{align*}
where $\epsilon_i(\bm{b}) := |p^A_i(\bm{b}) - p^\ast_i(\bm{b})|$ is the absolute error for $\bm{b}$, and by assumption $\epsilon_i(\bm{b}) \in \mathcal{O}(\delta)$, i.e., $\epsilon_i(\bm{b}) \leq C \delta$ for some constant $C$.

Now we bound each (fixed) $i$-summand. Note that, with $i$ fixed, $\epsilon_i(\bm{b}) \cdot \prod_{j<i} p^A_j(\bm{b})$ depends only on $\bm{b}_{\leq i} = (b_1,\dots,b_i)$, while $\prod_{j>i} p^\ast_j(\bm{b})$ depends on $\bm{b}_{>i} = (b_{i+1},\dots,b_k)$ (and on the parents of those nodes, which lie in $\bm{b}_{\leq i}$). The constraint $b_k = y$ touches only the second factor, so we may split the sum as:
\begin{align*}
    &\sum_{\bm{b}:\, b_k=y} \left(\epsilon_i(\bm{b}) \cdot \displaystyle\prod_{j<i} p^A_j(\bm{b}) \cdot \displaystyle\prod_{j>i} p^\ast_j(\bm{b}) \right)\\
    &= \sum_{\bm{b}_{\leq i}} \epsilon_i(\bm{b}) \cdot \displaystyle\prod_{j<i} p^A_j(\bm{b}) \cdot \sum_{\bm{b}_{>i}:\, b_k=y}\displaystyle\prod_{j>i} p^\ast_j(\bm{b})\\
    &\leq \sum_{\bm{b}_{\leq i}} \epsilon_i(\bm{b}) \cdot \displaystyle\prod_{j<i} p^A_j(\bm{b}),
\end{align*}
where the inner sum is bounded by $1$ because, by the local Markov property, $\sum_{\bm{b}_{>i}:\, b_k=y}\prod_{j>i} p^\ast_j(\bm{b}) = Pr_{\mathcal{M}^\ast}(y \mid \bm{b}_{\leq i})$, a probability. Then because $\epsilon_i(\bm{b}) \leq C \delta$,
\begin{align*}
    &\sum_{\bm{b}_{\leq i}} \epsilon_i(\bm{b}) \cdot \displaystyle\prod_{j<i} p^A_j(\bm{b})\\
    &\leq C\delta \cdot \sum_{\bm{b}_{\leq i}} \displaystyle\prod_{j<i} p^A_j(\bm{b})\\
    &= C \cdot |dom(B_i)| \cdot \delta,
\end{align*}
where the last equality holds because the integrand $\prod_{j<i} p^A_j(\bm{b})$ does not depend on $b_i$, so summing over $b_i$ contributes a factor of $|dom(B_i)|$, and $\sum_{\bm{b}_{<i}}\prod_{j<i} p^A_j(\bm{b}) = \sum_{\bm{b}_{<i}} Pr_{\mathcal{M}^A}(\bm{b}_{<i}) = 1$ as the sum over all values of a marginal. Setting $C' := C \cdot \max_i |dom(B_i)|$, the per-$i$ bound is $C'\delta$.

Summing the per-$i$ bound across the $k$ ancestral nodes of $Y$, the total error is of order $\mathcal{O}(k\delta)$, with implicit constant scaling like $\sum_i |dom(B_i)|$. For a fixed CID, both $k$ and the domain sizes are constants, so the bound simplifies to $\mathcal{O}(\delta)$; if one were to compare across CIDs of growing size, the implicit constant would grow accordingly. Analogous arguments give error bounds of $\mathcal{O}(\delta)$ for both the joint probability distribution $Pr_{\mathcal{M}^A}(Y, \bm{Pa}^D)$ and the prior $Pr_{\mathcal{M}^A}(\bm{Pa}^D)$.

So we have:

\begin{align*}
    Pr_{\mathcal{M}^A}(Y, \mathbf{Pa}^D)= Pr_{\mathcal{M}^*}(Y, \mathbf{Pa}^D) + \mathcal{O}(\delta)
\end{align*}
and
\begin{align*}
    Pr_{\mathcal{M}^A}(\mathbf{Pa}^D)= Pr_{\mathcal{M}^*}(\mathbf{Pa}^D) + \mathcal{O}(\delta)
\end{align*}

which will give us

\begin{align*}
    Pr_{\mathcal{M}^A}(Y \mid \mathbf{Pa}^D)= Pr_{\mathcal{M}^*}(Y \mid\mathbf{Pa}^D) + \mathcal{O}(\delta)
\end{align*}

when combined with the ratio formula for conditional probability and the following two facts about $\mathcal{O}$ (when $x \rightarrow 0)$:

\begin{align}
    (c + \mathcal{O}(x))(d+\mathcal{O}(x))&=cd + \mathcal{O}(x)\label{eq:multO}\\
    \frac{1}{c + \mathcal{O}(x)} &= \frac{1}{c}+\mathcal{O}(x)\label{eq:fracO}
\end{align}

\cref{eq:multO} is very straightforwardly verified. For \cref{eq:fracO}, assume $c > 0$ (or, more generally, that $c$ is bounded away from $0$ on the domain of interest) and consider that for $\epsilon \in \mathcal{O}(x)$,

\begin{align*}&\frac{1}{c + \epsilon} = \frac{1}{c}\left(\frac{1}{1 + \frac{\epsilon}{c}}\right) \\&= \frac{1}{c}\sum_{i=0}^\infty \left(-\frac{\epsilon}{c}\right)^i \\&= \frac{1}{c} \left(1 - \frac{\epsilon}{c} + \mathcal{O}(x^2)\right) \\&= \frac{1}{c}+\mathcal{O}(x),\end{align*}

where the geometric-series expansion is valid for $|\epsilon/c| < 1$, which holds for $x$ sufficiently small. Applied to $c = Pr_{\mathcal{M}^*}(\mathbf{Pa}^D = \mathbf{pa}^D)$, this requires the prior over parent-values to be bounded below on its support; the theorem's hypothesis $Pr_{\mathcal{M}^*}(\mathbf{Pa}^D = \mathbf{pa}^D) > 0$ guarantees positivity pointwise, and the uniform capability threshold below requires that this minimum, taken over all $\mathbf{pa}^D$ in the support, be strictly positive.

So we have, for every value $y$ of $Y$,
\begin{align*}
    Pr_{\mathcal{M}^A}(y \mid \mathbf{Pa}^D) = Pr_{\mathcal{M}^*}(y \mid \mathbf{Pa}^D) + \mathcal{O}(\delta).
\end{align*}

By guessability and \cref{lem:guessknow}, $Pr_{\mathcal{M}^*}(y^* \mid \mathbf{Pa}^D) = 1$ for the true value $y^*$ of $Y$, and hence $Pr_{\mathcal{M}^*}(y' \mid \mathbf{Pa}^D) = 0$ for every $y' \neq y^*$. The above bound therefore yields
\begin{align*}
    Pr_{\mathcal{M}^A}(y^* \mid \mathbf{Pa}^D) &\geq 1 - \mathcal{O}(\delta),\\
    Pr_{\mathcal{M}^A}(y' \mid \mathbf{Pa}^D) &\leq \mathcal{O}(\delta) \qquad \text{for every } y' \neq y^*.
\end{align*}

Choosing $\delta$ small enough that the $\mathcal{O}(\delta)$ term is below $1/2$ ensures $Pr_{\mathcal{M}^A}(y^* \mid \mathbf{Pa}^D) > Pr_{\mathcal{M}^A}(y' \mid \mathbf{Pa}^D)$ for every $y' \neq y^*$, so $y^*$ is the unique mode of $Pr_{\mathcal{M}^A}(\cdot \mid \mathbf{Pa}^D)$ and the agent's likeliest answer is the true one.

\end{proof}
    
\subsection{Proof of impossibility statements}

Before we prove the \enquote{impossibility} statements of section \ref{sec:general}, some more definitions are required.

While we used the notion of being robustly capable in a model on a set of distributions in an intuitive sense in section \ref{sec:general}, we now define it as follows:

\begin{definition}[Robust Capability on Distributions]\label{def:robustCapOnDist}
    Let $\mathcal{M}$ be a CID, $\mathcal{S}$ a set of distributions on $\mathcal{M}$, and $U$ a utility function. We call an agent $\Gamma$ acting in $\mathcal{M}$ \emph{robustly capable on $\mathcal{S}$ w.r.t. $U$} if, for every $\bm{\sigma} \in \mathcal{S}$, $\Gamma(\bm{\sigma})$ is an optimal policy in $\mathcal{M}_{\bm{\sigma}}$.
\end{definition}

\begin{definition}
    For a training algorithm $A$, an agent $\Gamma$ is such that \emph{A doesn't see any need to improve $\Gamma$} given dataset $\mathcal{D}$ if $A(\Gamma, \mathcal{D})=\Gamma$.
\end{definition}

\begin{definition}[Indifference of a Training Strategy]
    A training strategy $\mathcal{T}=(U, A)$ is indifferent between agents $\Gamma_1$ and $\Gamma_2$ if (for all the developers know) the training algorithm $A = (D, \alpha)$ may output either $\Gamma_1$ or $\Gamma_2$, i.e., for any initialised agent $\Gamma_0$, $\alpha(\mathcal{D}, \Gamma_0)$ may be either $\Gamma_1$ or $\Gamma_2$.
\end{definition}

\begin{example}
    An example of a training strategy $\mathcal{T}$ and agents $\Gamma_1, \Gamma_2$ such that $\mathcal{T}$ is \emph{not} indifferent between $\Gamma_1, \Gamma_2$ would be a training strategy that does not terminate unless a certain capability threshold had been reached (and perhaps outputs the initial agent $\Gamma_0$ if this is not done after a specified period of time). If $\Gamma_2$ clears the threshold and $\Gamma_1$ does not, this training strategy would not be indifferent between them.
\end{example}

\subsubsection{\cref{lem:freelunch}} \label{sec:freelunchlemmaproof}

\begin{proof}
Since $\mathcal{T}$ may output any robustly capable agent which optimises $U$ on $\mathcal{M}^*$, $\mathcal{S}$, we need to show that there is such a robustly capable agent which goal misgeneralises on new distributions.

$\mathcal{M}^*$ being goal-environment ambiguous on $\mathcal{S}$ between the training objective $U$ and $\tilde{U}$ means (by \cref{def:ambiguity}) that, for each $\bm{\sigma} \in \mathcal{S}$, $U$ and $\tilde{U}$ induce the same set of optimal policies in $\mathcal{M}^*$.

$\tilde{U}$ being substantially divergent from $U$ means that there is $\bm{\tilde{\sigma}} \in \bm{I}_{\mathcal{M}^*} \setminus \mathcal{S}$ such that on $\bm{\tilde{\sigma}}$, all policies which are optimal w.r.t. $\tilde{U}$ are very suboptimal w.r.t. $U$.

Let $\Gamma_M$ be an agent that is robustly capable w.r.t. $\tilde{U}$ on the entirety of $\bm{I}_{\mathcal{M}^*}$ and has a correct (up to the utility function) subjective model $\mathcal{M}^A$ (robust capability all but guarantees the latter if \cref{def:unmediated,def:domaindep} are satisfied, cf. \cref{thm:richens2}; but even if they are not, such an agent will exist).

By goal-environment ambiguity between $U$ and $\tilde{U}$ on $\mathcal{S}$, for each $\bm{\sigma} \in \mathcal{S}$, if $\Gamma_M(\bm{\sigma})$---that is, the policy $\Gamma_M$ outputs on $\bm{\sigma}$---is optimal w.r.t. $\tilde{U}$, it is also optimal w.r.t. $U$.

By robust capability w.r.t. $\tilde{U}$, the antecedent is indeed the case, i.e., $\Gamma_M$ outputs optimal policies w.r.t. $U$ on all of $\mathcal{S}$.

This means (\cref{def:robustCapOnDist}) that $\Gamma_M$ is robustly capable on $\mathcal{S}$ w.r.t. $U$.

But by substantial divergence, $\Gamma_M$ is very suboptimal w.r.t. $U$ on $\bm{\tilde{\sigma}}$.

That is, $\Gamma_M$ goal misgeneralises on $\bm{\tilde{\sigma}}$.

\end{proof}

\subsubsection{\cref{thm:MistakenFeedback}} \label{sec:learnablemistakesproof}

We model the evaluator by a node $E \in \bm{V}$. $E$ is a parent node of the agent's utility node $U$: $E$ decides what the correct answer to a question would have been based on $\mathbf{Pa}^E$, and $U(D, E)=1$ if $E=D$, and 0 otherwise.

\begin{definition}[Learnable Mistakes]\label{def:learnableMistakes}
    For a CID $\mathcal{M}^*$, distributions $\mathcal{S}$, evaluator node $E \in \bm{V}$ makes \emph{learnable mistakes} from the perspective of the decision node $D$ if there is a distribution $\bm{\sigma} \in \mathcal{S}$ and a variable $Y \in \bm{V}$ s.t. $\arg \max_{y \in dom(Y)}P_{\mathcal{M}^*_{\bm{\sigma}}}(E=y \mid Q=\text{{``Y"}}, \mathbf{Pa}^D) \neq \arg \max_{y \in dom(Y)}P_{\mathcal{M}^*_{\bm{\sigma}}}(Y=y \mid Q=\text{{``Y"}}, \mathbf{Pa}^D)$.
\end{definition}

That is, the true causal structure $\mathcal{M}^*$ (which includes the evaluation mechanism $E$) is such that, from the perspective of the agent, their best guess for what the truth is is different from what their best guess for what the evaluator will think the truth is.

(Note that if $E$ is infallible, i.e. it robustly tracks all the nodes about which questions can be asked---if $E$'s structural function always outputs the true value of the variable which the question indicated---this cannot happen, since $Pr_{\mathcal{M}^*}(Y)$ will always be equal to $Pr_{\mathcal{M}^*}(E \mid Q=\text{``Y"})$).

We now restate \cref{thm:MistakenFeedback} more explicitly:

\begin{theorem}
    For a training strategy $\mathcal{T}$ on a domain dependent, unmediated environment $\mathcal{M}^*$ containing evaluator node $E$ and training distributions $\mathcal{S}$:
    If $E$ makes learnable mistakes on $\mathcal{S}$ then every robustly capable agent on some $\bm{\sigma} \in \mathcal{S}$ w.r.t. the training utility (given by \cref{eq:EvaluatorUtility}) is dishonest---it is an evaluation simulator.
\end{theorem}

\begin{proof}
    Let $\Gamma$ be a robustly capable agent in a CID $\mathcal{M}^*$ with evaluator node $E$ such that $E$ makes learnable mistakes.

    Let $Y$ and $y_E$ be as in \cref{def:learnableMistakes}, i.e. when asked about $Y$, $y_E = \arg \max_{y \in dom(Y)}P_{\mathcal{M}^*_{\bm{\sigma}}}(E=y \mid Q=\text{{``Y"}}, \mathbf{Pa}^D)$ but $y_E \neq y^*=\arg \max_{y \in dom(Y)}P_{\mathcal{M}^*_{\bm{\sigma}}}(Y=y \mid Q=\text{{``Y"}}, \mathbf{Pa}^D)$.

    However, since the agent's utility depends on what the evaluator says; and specifically, is maximised if its own answer matches the evaluator's, its optimal policy will answer $y_E$ instead of $y^*$---a dishonest answer.

    Since $\Gamma$ is robustly capable, it outputs the optimal policy. Thus, its answer is $y_E$, not $y^*$. By \cref{thm:honesttruth}, this is not what it believes. It is a dishonest evaluation simulator.
\end{proof}

\subsubsection{\cref{thm:ELKStrikesBack}} \label{sec:elkstrikesbackproof}

\begin{proof}
    We show that there are two substantially divergent utility functions $U$ and $\tilde{U}$ such that the training environment is ambiguous (\cref{def:ambiguity}) between them, and that there is a $\bm{\tilde{\sigma}} \notin \mathcal{S}$ such that the optimal policy w.r.t. $U$ is to be an honest agent, but the optimal policy w.r.t. $\tilde{U}$ is an evaluation simulator.

    Then \cref{thm:ELKStrikesBack} follows from \cref{lem:freelunch}.

    If the evaluator never makes mistakes in the training environment, that is, if $\arg\max_y Pr_{\mathcal{M}^*_\sigma}(Y=y \mid Q=\text{Y'}, \mathbf{Pa}^D)$ is always equal to $\arg\max_y Pr_{\mathcal{M}^*_\sigma}(E=y \mid Q=\text{`Y'}, \mathbf{Pa}^D)$, for all $\bm{\sigma}\in \mathcal{S}$ then $\mathcal{M}^*$ with distributions $\mathcal{S}$ is ambiguous between the following two utility functions:
\begin{align*}
U(D, E, Q)&=\begin{cases}
1 & \text{if } Q=\text{``Y"},\ D=y^*, \text{ and for all } \hat{y}: \\
& \text{Pr}_{\mathcal{M}^*}(Y=y^* \mid \mathbf{Pa}^{D}) \geq \text{Pr}_{\mathcal{M}^*}(Y=\hat{y} \mid \mathbf{Pa}^{D}), \\
0 & \text{otherwise.}
\end{cases}
\\
\tilde{U}(D, E, Q)&=\begin{cases}
1 & \text{if } D=E\\
0 & \text{otherwise.}
\end{cases}
\end{align*}

That is, if the training objective corresponds to the judgments of a (perfect) evaluator, then the environment is ambiguous between two utility functions: the one which rewards saying what is true, and the one which rewards saying what the evaluator belives to be true.

However, by assumption, there is $\tilde{\bm{\sigma}}\in \bm{I}_{\mathcal{M}^*} \setminus \mathcal{S}$ such that on $\tilde{\bm{\sigma}}$, the evaluator does make mistakes. That is, there will be $Y \in \bm{V}$ and values $y^*$, $y_E$, such that $y_E = \arg \max_{y \in dom(Y)}P_{\mathcal{M}^*_{\tilde{\bm{\sigma}}}}(E=y \mid Q=\text{{``Y"}}, \mathbf{Pa}^D) \neq y^*=\arg \max_{y \in dom(Y)}P_{\mathcal{M}^*_{\tilde{\bm{\sigma}}}}(Y=y \mid Q=\text{{``Y"}}, \mathbf{Pa}^D)$.

So, on $\tilde{\bm{\sigma}}$, answering $y^*$ to $Q=\text{``Y"}$ will give $U= 1$, whereas answering $y_E$ will give $U = 0$; with the situation being reversed for $\tilde{U}$. Thus, $U$ and $\tilde{U}$ diverge substantially; and the optimal policy w.r.t. $U$ is to be an honest agent, whereas the optimal policy w.r.t. $\tilde{U}$ is an evaluation simulator. 
\end{proof}

\end{document}